\documentclass[twoside,11pt]{article}

\usepackage{jmlr2e}

\usepackage[utf8]{inputenc}  

\usepackage{placeins}  

\usepackage{amsmath}

\usepackage{bm}
\usepackage{bbm}

\DeclareMathOperator*{\vvec}{vec}
\DeclareMathOperator{\E}{\mathbb{E}}
\newcommand{\x}{\mathbf{x}}
\newcommand{\z}{\mathbf{z}}

\newcommand{\D}{\mathbf{D}}

\newcommand{\T}{\mathbf{T}}
\newcommand{\M}{\hat{\mathbf{M}}}
\newcommand{\OO}{\mathcal{O}}
\newcommand{\X}{\mathcal{X}}

\newcommand{\LY}{\mathcal{L(Y)}}
\newcommand{\R}{\mathbb{R}}

\usepackage{csquotes}
\usepackage{enumerate}

\usepackage{tikz}

\usepackage{algorithm}
\usepackage[noend]{algorithmic}

\usepackage{array,multirow}

\usepackage{colortbl}
\usepackage{booktabs}

\usepackage{subfigure}

\usepackage{lastpage}
\jmlrheading{22}{2021}{1-\pageref{LastPage}}{8/19; Revised
10/20}{1/21}{19-665}{Riikka Huusari and Hachem Kadri}
\ShortHeadings{Entangled Kernels}{Huusari and Kadri}

\firstpageno{1}

\begin{document}

\title{Entangled Kernels - Beyond Separability}

\author{\name Riikka Huusari \email riikka.huusari@aalto.fi \\
       \addr Helsinki Institute for Information Technology HIIT \\
       Department of Computer Science\\ 
       Aalto University\\
       02150 Espoo, Finland
       \AND
       \name Hachem Kadri \email hachem.kadri@lis-lab.fr \\
       \addr Department of Computer Science\\
       Aix-Marseille University, CNRS, LIS\\
       13013 Marseille, France}

\editor{Arthur Gretton}

\maketitle


\begin{abstract}
We consider the problem of operator-valued kernel learning and investigate the possibility of going beyond the well-known separable kernels.
Borrowing tools and concepts from the field of quantum computing, such as partial trace and entanglement, we propose a new view on operator-valued kernels and define a general family of kernels that encompasses previously known operator-valued kernels, including separable and transformable kernels.
Within this framework, we introduce another novel class of operator-valued kernels called \textit{entangled kernels} that are not separable.
We propose an efficient two-step algorithm for this framework, where the entangled kernel is learned based on a novel extension of kernel alignment to operator-valued kernels. 
We illustrate our algorithm with an application to supervised dimensionality reduction, and demonstrate its effectiveness with both artificial and real data for multi-output regression.
\end{abstract}

\smallskip 




\section{Introduction}
\label{sec:intro}

There is a growing body of learning problems for which each instance in the training set is naturally associated with a set of discrete and/or continuous labels~\citep{izenman1975reduced, caruana1997multitask, Micchelli2005onlearning, alvarez2011computationally, dembczynski2012label, baldassarre2012multi}. Output kernel learning algorithms approach these problems by learning simultaneously a vector-valued function in a reproducing kernel Hilbert space~(RKHS) and a positive semi-definite matrix that describes the relationships between the labels~\citep{dinuzzo2011learning,dinuzzo2011learning-b,Ciliberto2015convex,jawanpuria2015efficient}.
The main idea of these methods is to learn a separable operator-valued kernel.

Operator-valued kernels appropriately generalize the well-known notion of reproducing kernels and provide a means for extending the theory of reproducing kernel Hilbert spaces from scalar- to vector-valued functions. They were introduced as a machine learning tool in~\cite{Micchelli2005onlearning} and have since been investigated for use in various machine learning tasks, including multi-task learning~\citep{Evgeniou2005learning}, functional regression~\citep{Kadri2015operator}, structured output prediction~\citep{Brouard2016input}, quantile learning~\citep{sangnier2016joint}, multi-view learning~\citep{Minh2016unifying} and reinforcement learning~\citep{lever2016compressed}. 
The kernel function evaluated on two data samples in this setting outputs a linear operator~(a matrix in the case of finite-dimensional output spaces, $K(x,z)\in\R^{p\times p}$ with $p$ the dimension of the output space) which encodes information about multiple output variables. 
A challenging question in vector-valued learning is what sort of interactions should the operator-valued kernel learn and quantify, and how should one build and design these kernels. This is the main question investigated in the paper in the context of non-separability between input and output variables.

Some classes of operator-valued kernels have been proposed in the literature~\citep{Caponnetto2008universal,alvarez2012kernels}, with separable  kernels being one of the most widely used  for learning vector-valued functions due to their simplicity and computational efficiency. These kernels are formulated as a product between a kernel function for the input space alone, and a matrix that encodes the interactions among the outputs. Indeed, the name of the class refers to the fact that dependencies between input and output variables are considered separately.
In order to overcome the need for choosing a kernel before the learning process, output kernel learning methods learn the output matrix from data~\citep{dinuzzo2011learning,Ciliberto2015convex,jawanpuria2015efficient}. 
However there are limitations in using separable kernels. These kernels use only one output matrix and one input kernel function, and then cannot capture different kinds of dependencies and correlations. Moreover the kernel matrix associated to separable operator-valued kernels is a rank-one kronecker product matrix~(i.e, computed by only one kronecker product $\mathbf{K} \otimes \mathbf{T}$, where $\mathbf{K}$ is the scalar-valued kernel matrix and $\mathbf{T}$ is the output similarity matrix), which is restrictive as it assumes a strong repetitive structure in the operator-valued kernel matrix that models input and output interactions as illustrated in Figure~\ref{fig:separableIllustration1}.

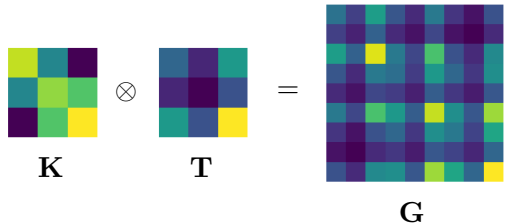
\begin{figure}[tb]
\centering
\begin{tikzpicture}
\node at (0,0) {\includegraphics[width=.1\textwidth]{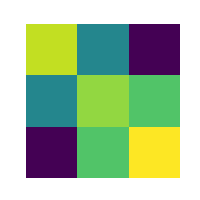}};
\node at (0,-1) {$\mathbf{K}$};
\node at (1,0) {$\otimes$};
\node at (2,0) {\includegraphics[width=.1\textwidth]{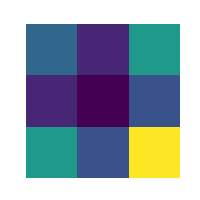}};
\node at (2,-1) {$\mathbf{T}$};
\node at (4.8,0) {\includegraphics[width=.2\textwidth]{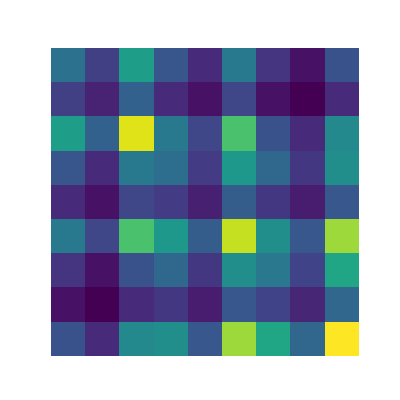}};
\node at (4.8,-1.6) {$\mathbf{G}$};
\node at (3.15,0) {$=$};
\end{tikzpicture}
\caption{Illustration on the restrictiveness of a separable kernel. $\mathbf{K}$ is the scalar-valued kernel matrix, $\mathbf{T}$ is the output similarity matrix and $\mathbf{G}$ is the operator-valued kernel matrix.
	Every block of the big kernel matrix $\mathbf{G} = \mathbf{K} \otimes \mathbf{T}$ has the same structure, and thus models the output dependencies almost the same no matter the input interactions in $\mathbf{K}$.}\label{fig:separableIllustration1}
\end{figure}

To go beyond separable kernels, some attempts have been made to learn a weighted sum of them in the multiple kernel learning framework~\citep{kadri2012multiple,sindhwani2013scalable,gregorova2017forecasting}.  Another approach, proposed by~\cite{lim2015operator}, is to learn a  combination of a separable and a transformable kernel, the latter being a type of non-separable kernel based on representing the data via label-dependent transformations. 
In that work, the form of the transformable kernel is fixed in advance but allows to encode non-separable dependencies between inputs and outputs.
Despite these previous investigations, the lack of knowledge about the full potential of operator-valued kernels and how to go beyond the restrictive separable kernel clearly hampers their widespread use in machine learning and other fields.

This paper deals with the problem of learning non-separable kernels. It is a significant extension of our previous conference paper~\citep{huusari2019entangled}, giving more thorough treatment of the background material, additional theoretical results, full proofs, and more insights to the developed framework. It also provides a theoretical analysis of the generalization error of the learning method, along with expanded experimental section.
Our main contributions are:
\begin{itemize}
\item By leveraging tools from the field of quantum computing, we  introduce a novel class of kernels based on the notion of partial trace which generalizes the trace operation to block matrices.
This class of partial trace kernels we propose is very
broad and encompasses previously known operator-valued kernels, including separable and transformable kernels, which we illustrate with examples. 

\item From the new class of partial-trace kernels we derive another new class of operator-valued kernels, called \textit{entangled} kernels, that are not separable. As far as we are aware, this is the first time such an operator-valued kernel categorization has been performed. 

\item We further study this class of kernels and develop a new algorithm called EKL~(Entangled Kernel Learning) that in two steps learns an entangled kernel and a vector-valued function. For the first step of kernel learning, we propose a novel definition of alignment between an operator-valued kernel and labels of a multi-output learning problem. To our knowledge, this is the first proposition on how to extend alignment to the context of operator-valued kernels. Our algorithm offers improvements to the high computational cost usually associated with learning with general operator-valued kernels. 

\item  We prove a bound on the generalization error of our method  using the notion of Rademacher complexity.

\item We provide an empirical evaluation of EKL. First, we illustrate how EKL works by applying it to the task of supervised dimensionality reduction in the multi-task setting. We also thoroughly study its performance and demonstrate its effectiveness on artificial data as well as real benchmarks. Finally we compare the running times of learning with various classes of operator-valued kernels.

\end{itemize}

The  remainder  of  this  paper  is  organized  as  follows. We begin in~Section~\ref{sec:quantum} with a short background on quantum entanglement and learning with operator-valued kernels. In Section~\ref{sec:learningOvK}, we describe some known classes of operator-valued kernels and review previous work on learning separable operator-valued kernels. Section~\ref{sec:ptk} then introduces the new classes of \textit{partial-trace} and \textit{entangled} kernels. Our new algorithm EKL for learning entangled kernels is given in Section~\ref{sec:ekl}, along with generalization analysis. In Section~\ref{sec:xp}, we present our experimental results for both synthetic and real-life data. We conclude in Section~\ref{sec:conclusion} and present some technical details in the appendix.

\subsection{Notation}

We denote scalars, vectors and matrices as $a$, $\mathbf{a}$ and $\mathbf{A}$ respectively. 
The notation $\mathbf{A} \geq 0$ will be used to denote a positive semi-definite~(psd) matrix.
Throughout the paper we use $n$ as the number of labeled data samples 
and $p$ as the number of outputs corresponding to one data sample. 
We denote our set of data samples by  $\{x_i,y_i\}_{i=1}^n$ on $\mathcal{X} \times \mathcal{Y}$, where $\mathcal{X}$ is a Polish space and $\mathcal{Y}$ is a separable Hilbert space. Usually, $\mathcal{X}$ and $\mathcal{Y}$ are respectively $\R^d$ and $\R^p$ equipped with the standard Euclidean metric. 
Without loss of generality, we can assume that $\mathcal{X}=\R^d$ and $\mathcal{Y}=\R^p$, and thus denote our data set as $\{\mathbf{x}_i,\mathbf{y}_i\}_{i=1}^n$.
We use $k(\cdot, \cdot)$ as a scalar-valued, and $K(\cdot, \cdot)$ as an operator-valued kernel function; the corresponding kernel matrices are $\mathbf{K} \in \R^{n\times n}$ and $\mathbf{G}\in \R^{np\times np}$, the latter containing blocks of size $p\times p$. 
We denote by $\mathcal{K}$ and $\mathcal{H}$ the reproducing kernel Hilbert spaces~(RKHS) associated to the kernels $k$ and $K$, respectively.
Table~\ref{tab:notation} summarizes the notation used in this paper.
	
	\begin{table}
		\label{tab:notation}
		\centering
		\def\arraystretch{1.5}%
		\begin{tabular}{lp{5cm}||lp{5cm}}
		$\mathcal{X}$ & input space & $\mathcal{Y}$ & output space\\
		$k(\cdot,\cdot)$ & scalar-valued kernel & $K(\cdot,\cdot)$ & operator-valued kernel  \\
		$\mathcal{K}$ & reproducing kernel Hilbert space of $k$  & $\mathcal{H}$ & reproducing kernel Hilbert space of $K$ \\
		 $\mathbf{K}$& the kernel matrix of $k$ & $\mathbf{G}$ & the (block) kernel matrix of $K$ \\
		 $\phi$ & feature map of $k$ (from $\mathcal{X}$ to $\mathcal{Y}$)  & $\Gamma$ & feature map of $K$ (from $\mathcal{X}$ to $\mathcal{L(Y,H)})$ \\
			$\mathcal{L(A,B)}$ & the set of trace-class operators from $\mathcal{A}$ to $\mathcal{B}$ & $\mathcal{L(A)}$ & the set $\mathcal{L(A,A)}$ \\
			$\mathbf{A}^\top$, $\mathbf{u}^\top$ & the transpose of a matrix $\mathbf{A}$ or a vector $\mathbf{u}$ & 	$\mathbf{A}^*$, $\mathbf{u}^*$ & the adjoint of an operator $\mathbf{A}$ or a vector $\mathbf{u}$ \\
			$\mathbf{A}\geq 0$ & a positive semi-definite~(psd) matrix & A($\mathbf{A}$,$\mathbf{B}$) & alignment between matrices $\mathbf{A}$ and $\mathbf{B}$  \\
			$\mathcal{A}_1 \otimes \mathcal{A}_2$ & the tensor product of Hilbert spaces $\mathcal{A}_1$ and $\mathcal{A}_2$ & $\mathbf{A}_1 \otimes \mathbf{A}_2$ & the tensor product of operators $\mathbf{A}_1$ and $\mathbf{A}_2$ \\
			$\mathrm{tr}(\mathbf{A})$ & the trace of a matrix or an operator $\mathbf{A} \in \mathcal{L(A)}$ & $\mathrm{tr}_{\mathcal{A}_2}(\mathbf{A})$& the partial trace of a matrix or an operator $\mathbf{A} \in \mathcal{L}(\mathcal{A}_1\otimes \mathcal{A}_2)$\\
		\end{tabular}
	\def\arraystretch{1}%
		\caption{Notation summary.}
	\end{table} 

\section{Background}
\label{sec:quantum}

We now give some background about quantum entanglement, and review the basics of learning with operator-valued kernels.

\subsection{Quantum Entanglement}\label{section:quantum}

The field of quantum computing is vast, and rapidly growing. 
This section is not intended to provide a broad overview or exhaustive survey of the literature on quantum etanglement, but gives some notions on entanglement as a quantum property of mixed composite quantum systems that inspired our entangled kernel design. We refer the reader to~\cite{horodecki2009quantum}, ~\cite{bengtsson2017geometry}, and~\citet[chap.~10]{rieffel2011quantum} for more background information. We will now start with very basics of quantum computation, for this we refer the reader to~\citet[chap.~2-3]{rieffel2011quantum}.

A major difference of quantum computing and quantum information theory to their classical counterparts is that instead of bits the \enquote{particles} carrying information are qubits. 
Unlike bits which have only two possible states, 0 and 1, qubits can exist in those and any combination of them. More formally, a qubit takes values $a\psi_0 + b\psi_1$ where $\psi_0$ and $\psi_1$ are orthonormal basis vectors and $a,b\in\mathbb{C}$ such that $|a|^2+|b|^2 = 1$.\footnote{In the field of quantum computing, it is more usual to use the Dirac's bra-ket notation for the basis vectors. In this notation \enquote{bra} $\langle x |$ denotes a row vector and \enquote{ket} $|x\rangle$ a column vector, and a qubit would take values $a\vert 0 \rangle + b\vert 1 \rangle$.} 
In quantum information theory the actual state $a\psi_0 + b\psi_1$ cannot be recovered by measurement; it is always measured as either  $\psi_0$ or $\psi_1$ according to the probabilities proportional to multipliers $a$ and $b$.

In the heart of quantum computing there is a notion of quantum systems, consisting of one or more qubits. For a system of one qubit, the system's basis consists of two-dimensional vectors. For a system of $n$ qubits, the description requires a $2^n$-dimensional Hilbert space to capture all possible combinations of the qubit values. 
This brings forward the notion of entanglement; any state that cannot be written as a tensor product of $n$ single-qubit states is said to be entangled. Perhaps the simplest example of an entangled state is
\begin{equation}
\frac{1}{\sqrt{2}}\left( \psi_{00} + \psi_{11} \right)
\end{equation}
as it cannot be written as 
\begin{equation*}
\left( a_1\psi_0 + b_1\psi_1 \right) \otimes \left( a_2\psi_{0} + b_2\psi_1 \right) =  a_1a_2 \psi_{00} + a_1b_2 \psi_{01} + b_1a_2\psi_{10} + b_1b_2\psi_{11}
\end{equation*}
with any multipliers $a_1$, $a_2$, $b_1$ and $b_2$, and where $\psi_{01}=\psi_0\otimes \psi_1$, similarly for others. 

A quantum system exists in a state, describing all the information that can be learned of the system. 
A quantum system can also be divided into parts or subsystems. 
We focus here only on \textit{bipartite} quantum systems, i.e., systems composed of two distinct subsystems. The Hilbert space $\mathcal{F}$ associated with a {bipartite} quantum system is given by the tensor product $\mathcal{F}_1 \otimes \mathcal{F}_2$ of the spaces $\mathcal{F}_1$ anf $\mathcal{F}_2$ corresponding to each of the subsystems. 
A question to ask in this context is, what information of the system can be obtained by only considering a part of it, either $\mathcal{F}_1$ or $\mathcal{F}_2$? 
A quantum state can be either \enquote{pure} or \enquote{mixed}.
While states of pure systems can be represented by a state vector $\psi \in \mathcal{F}$, for mixed states the characterization is done with density operators (or matrices) $\rho$, positive Hermitian operators with trace equal to one. Pure states can also be modeled with a density operator allowing for uniform treatment, in this case $\rho = \psi\psi^\top$. 

The entanglement present in a bipartite quantum system can be modeled through the \textit{partial trace} of the system. Given the density operator $\rho$ modeling the whole system, the state of, say, the first subsytem is described by a reduced density matrix, given by taking the {partial trace} of $\rho$ over $\mathcal{F}_2$. If the system can be accurately represented with only the two subsystems, then it is not entangled. In the following we review the notions of partial trace, separability and entanglement of bipartite quantum systems in more detail.

We denote the set of bounded linear operators from a Hilbert space $\mathcal{B}$ to $\mathcal{B}$ with finite trace norm as $\mathcal{L(B)}$.
Let $\mathcal{F}_1$ and $\mathcal{F}_2$ be separable Hilbert spaces.

\begin{definition}\label{de:partial_trace} (partial trace) \\[0.1cm] 
Let 
$\{e_i\}_i$ be an orthonormal basis for $\mathcal{F}_2$.
For an operator $\mathbf{A}$ in $\mathcal{L}(\mathcal{F}_1\otimes \mathcal{F}_2)$ its partial trace, $\mathrm{tr}_{\mathcal{F}_2} \mathbf{A}$, is an operator in $\mathcal{L}(\mathcal{F}_1)$ defined by the relation 
\begin{equation*}
\langle x, (\mathrm{tr}_{\mathcal{F}_2} \mathbf{A}) y \rangle = \sum_{i} \langle x\otimes e_i , \mathbf{A} (y\otimes e_i) \rangle
\end{equation*}
for all $x,y\in \mathcal{F}_1$. 

\end{definition}  

This definition follows the ones in~\citet[Equation~4.34]{bhatia2009positive} and~\citet[Equation~2.10]{attalLecturesPtr}. 
In the finite-dimensional case where $\mathcal{F}_1 = \mathbb{R}^p$ and $\mathcal{F}_2 = \mathbb{R}^N$, the operator $\mathbf{A} \in \mathbb{R}^{pN \times pN}$ is a block matrix where each block is of size $N \times N$, and the partial trace is obtained by computing the trace of each block~(see Figure~\ref{fig:pt}). To see this, let us denote the orthonormal bases of $\mathcal{F}_1$ and $\mathcal{F}_2$ by $\{g_j\}_{j=1}^{p}$ and $\{e_i\}_{i=1}^{N}$, respectively. 
Any block matrix $\mathbf{A} \in \mathbb{R}^{pN\times pN}$ can be written as $\mathbf{A} = \sum_{s,t=1}^p g_s g_t^\top \otimes \mathbf{A}_{st}$, where $\mathbf{A}_{st}\in\mathbb{R}^{N\times N}$ is the $(s,t)$th block of $\mathbf{A}$.
Now, when investigating one element of the operator $\mathrm{tr}_{\mathcal{F}_2} \mathbf{A}$ at position $(l,m)$ we see that it exactly corresponds to the trace of block $\mathbf{A}_{lm}$:
\begin{align*} [\mathrm{tr}_{\mathcal{F}_2} \mathbf{A}]_{lm} &= \langle g_l, (\mathrm{tr}_{\mathcal{F}_2} \mathbf{A}) g_m \rangle = \sum_{i=1}^N \langle g_l\otimes e_i , \mathbf{A} (g_m\otimes e_i) \rangle \\
&= \sum_{i=1}^N \sum_{s,t=1}^p \langle g_l\otimes e_i , (g_s g_t^\top \otimes \mathbf{A}_{st}) (g_m\otimes e_i) \rangle = \sum_{i=1}^N \sum_{s,t=1}^p \langle g_l\otimes e_i , g_s g_t^\top g_m \otimes \mathbf{A}_{st} e_i \rangle \\
&= \sum_{i=1}^N \sum_{s=1}^p \langle g_l\otimes e_i , g_s \otimes \mathbf{A}_{sm} e_i \rangle = \sum_{i=1}^N \sum_{s=1}^p \langle g_l, g_s\rangle \langle e_i, \mathbf{A}_{sm} e_i \rangle = \sum_{i=1}^N  \langle e_i, \mathbf{A}_{lm} e_i \rangle \\
&=tr(\mathbf{A}_{lm}).
\end{align*}
The last equality is easy to see from the definition of trace.
The following theorem shows how to compute partial trace for separable operators~\citep[Theorem 2.29]{attalLecturesPtr}.
\begin{theorem}\label{th:ptr_separable} (partial trace of a tensor product of  operators) \\[0.1cm] 
Let $\mathbf{B}$ and $\mathbf{C}$ be operators in $\mathcal{L}(\mathcal{F}_{1})$ and $\mathcal{L}(\mathcal{F}_{2})$, respectively. If $\mathbf{A}$ is an operator in $\mathcal{L}(\mathcal{F}_{1} \otimes \mathcal{F}_2)$ of the form $\mathbf{B} \otimes \mathbf{C}$, then
\begin{equation*}
\mathrm{tr}_{\mathcal{F}_2}(\mathbf{A}) = \mathbf{B}\ \mathrm{tr}(\mathbf{C}).
\end{equation*}
\end{theorem}

The notion of \textit{partial trace} is a generalization of the trace operation to block structured matrices~\citep[chap.~10]{rieffel2011quantum}.  
Note that there are two ways of generalizing trace to block matrices. Another possibility would be the so-called block trace~\citep{filipiak2018properties} which, informally, is defined as a sum of the diagonal blocks of a matrix; with $\mathbf{A}\in\R^{pN\times pN}$ it would be the sum $\sum_{t=1}^{p} \mathbf{A}_{tt}$ in which each $\mathbf{A}_{tt}$ is of size $N\times N$. However in this work we only consider the \enquote{blockwise trace} definition we discussed earlier.


\begin{figure}[tb]
	\centering
	\includegraphics[scale=0.25]{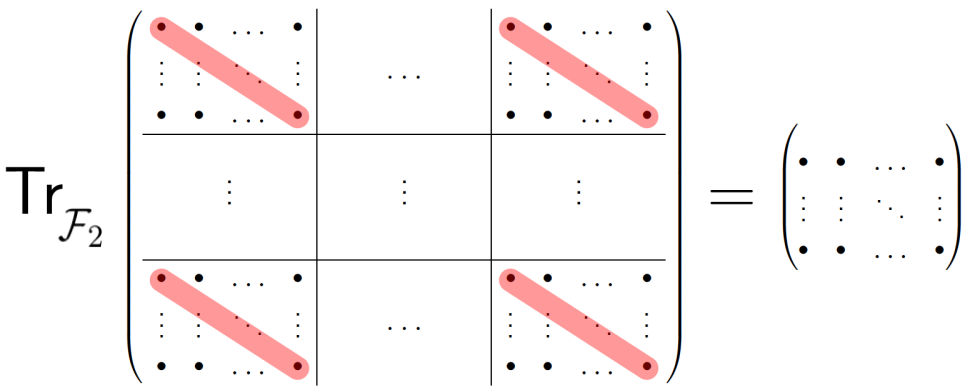}
	\caption{Illustration of partial trace operation. The partial trace operation applied to $N\times N$-blocks of a $pN\times pN$ matrix gives a $p\times p$ matrix as an output. }
	\label{fig:pt}
\end{figure}


In the case where the density matrix $\rho$ of a mixed bipartite state can be written as $\rho = \rho_1 \otimes \rho_2$, where $\rho_1$ and $\rho_2$ are density matrices on $\mathcal{F}_1$ and $\mathcal{F}_2$ of the subsystems, the partial trace of $\rho$ with respect to $\mathcal{F}_2$ is $\rho_1$. This form of mixed product states is restrictive and does not exhibit correlations between the two subsystems. A convex sum of different product states, 
\begin{equation}
\label{eq:qsep}
\rho = \sum_i p_i \ \!\rho_1^i \otimes \rho_2^i,
\end{equation}
with $p_i \geq0$ and $\sum_i p_i = 1$, however, will in general  represent certain types of correlations between the subsystems of the composite quantum system. These correlations can be described in terms of the classical probabilities $p_i$, and are therefore considered classical. States of the form (\ref{eq:qsep}) thus are called \textit{separable} mixed states. In contrast, a mixed state is \textit{entangled}  if it cannot be written as a convex combination of product states, i.e., 
\begin{equation}
\nexists \ \rho_1^i, \rho_2^i, p_i \geq 0 \quad \text{such that} \quad \rho = \sum_i p_i \ \! \rho_1^i \otimes \rho_2^i.
\end{equation}
Entangled states are one of the most commonly encountered classes of bipartite states possessing quantum correlations~\citep{mintert2009basic}.

A challenging problem in quantum computing is to identify necessary and sufficient conditions for quantum separability. Given a density matrix $\rho$ of a bipartite quantum state, the quantum separability problem asks whether $\rho$ is entangled or separable.
%
A useful and efficient necessary condition for checking if a given block density matrix is separable in some block size partition, is to use \textit{positive partial transpose~(PPT)} condition, sometimes also called Peres-Horodecki criterion~\citep{peres1996separability,Horodecki:321722}. The partial transpose of a $pN\times pN$ block matrix $\mathbf{P}$ with blocks $(\mathbf{P}_{ij})_{i,j=1}^p$ is the block matrix of the same size containing the transposed blocks $(\mathbf{P}_{ij}^\top)_{i,j=1}^p$. 
If a density matrix is separable, then it has positive partial transpose. It is necessary for any separable density matrix to have positive partial transpose; yet in general this condition is not sufficient in guaranteeing separability, as there might be non-separable density matrices fulfilling the PPT condition. However this condition guarantees that if the partial transpose matrix has a negative eigenvalue, the state is entangled.

\subsection{Learning with Operator-valued Kernels}

We now review the basics of operator-valued kernels~(OvKs) and their associated vector-valued reproducing kernel Hilbert spaces~(RKHSs) in the setting of supervised learning.
Vector-valued RKHSs were introduced to the machine learning community by~\cite{Micchelli2005onlearning} as a way to extend kernel machines from scalar to vector outputs. 
Given a set of training samples $\{\mathbf{x}_i,\mathbf{y}_i\}_{i=1}^n$ on $\mathcal{X} \times \mathcal{Y}$, the optimization problem 
 \begin{equation}
 \label{eq:minRKHS}
\arg\min_{f\in\mathcal{H}} \sum_{i=1}^n V(\mathbf{y}_i,f(\mathbf{x}_i)) + \lambda \|f\|_\mathcal{H}^2,
 \end{equation}
where $f$ is a vector-valued function and $V:\mathcal{Y}\times \mathcal{Y} \to \mathbb{R}_+$ is a convex loss function, can be solved in a vector-valued RKHS $\mathcal{H}$ by the means of a vector-valued extension of the representer theorem. 
\vspace{-0.0cm}
\begin{definition} (vector-valued RKHS) \\[0.1cm] 
  A Hilbert space $\mathcal{H}$ of functions from $\mathcal{X}$ to
  $\mathcal{Y}$ is called a reproducing kernel Hilbert space if
  there is a positive semi-definite $\LY$-valued kernel
  $K$ on $\mathcal{X} \times \mathcal{X}$ such that:
  \begin{enumerate}[i.]
    \item \label{enum1:i} the function $\z \mapsto K(\x,\z)\mathbf{y}$ belongs to $\mathcal{H},\ \forall \;\z, \x \in \mathcal{X},\ \mathbf{y} \in \mathcal{Y}$,
    \item \label{enum1:ii} $\forall f \in \mathcal{H}, \x \in \mathcal{X},\ \mathbf{y} \in \mathcal{Y}, \ \ 
      \langle f,K(\x,\cdot)\mathbf{y}\rangle _{\mathcal{H}} =
      \langle f(\x),\mathbf{y}\rangle_{\mathcal{Y}}$ \hspace*{0.1cm} (reproducing property).
  \end{enumerate}
\end{definition}
\begin{definition} (positive semi-definite operator-valued kernel) \\[0.1cm]
  A $\LY$-valued kernel $K$ on 
  $\mathcal{X}\!\!\,\times\!\!\,\mathcal{X}$ is a function
  $K(\cdot,\cdot):\mathcal{X} \times \mathcal{X}
  \rightarrow \LY$; it is positive semi-definite if:
  \begin{enumerate}[i.]
    \item $K(\x, \z)=K(\z, \x)^{*}$, where superscript $^*$ denotes the adjoint operator, 
    \item  and, for every $n\in\mathbb{N}$ and all
      $\{(\x_{i},\mathbf{y}_{i})_{i=1,\ldots ,n}\}\in \mathcal{X} \times
      \mathcal{Y}$,  \[\sum_{i,j} \langle
      \mathbf{y}_i, K(\x_{i},\x_{j})\mathbf{y}_{j}\rangle_{\mathcal{Y}} \geq 0.\]
  \end{enumerate}
\end{definition}

\begin{theorem} (bijection between vector-valued RKHS and positive semi-definite operator-valued kernel) \\ [0.1cm]
  \label{th:positiveDefiniteImpliesKernel}
  An $\LY$-valued kernel $K$ on
  $\mathcal{X}\times \mathcal{X}$ is the reproducing kernel of some Hilbert space
  $\mathcal{H}$, if and only if it is positive semi-definite.
\end{theorem}
\begin{theorem} (representer theorem) \\ [0.1cm]
  \label{th:representer}
Let $K$ be a positive semi-definite operator-valued kernel and $\mathcal{H}$ its corresponding vector-valued RKHS. The solution $\hat{f} \in \mathcal{H}$ of the regularized optimization problem~(\ref{eq:minRKHS}) has the following form
\begin{equation}
\hat{f}(\x) = \sum_{i=1}^n K(\x,\mathbf{x}_i)\mathbf{c}_i,\;\;\; \text{with} \;\;\; \mathbf{c}_i\in \mathcal{Y}.
\end{equation}

\end{theorem}
With regard to the classical representer theorem, here the kernel $K$ outputs a matrix and
the “weights” $\mathbf{c}_i$ are vectors. The proofs of Theorem~\ref{th:positiveDefiniteImpliesKernel} and~\ref{th:representer} can be found in~\cite{Micchelli2005onlearning} and~\cite{Kadri2015operator}. 
For further reading on operator-valued kernels and their associated RKHSs, see, e.g.,~\cite{Caponnetto2008universal,Carmeli2010vector, alvarez2012kernels}.

\section{Learning Operator-valued Kernels}
\label{sec:learningOvK}

In this section we first review some known classes of operator-valued kernels, before moving on to describing ways to learn them. Most of the works in this field consider separable kernels, but a few specialized methods exist also for non-separable kernels.

\subsection{Known Classes of Operator-valued Kernels}

Some well-known classes of operator-valued kernels include separable and transformable kernels. 
Note that here and throughout the rest of the manuscript we consider the case where the output space $\mathcal{Y}$ is of finite dimension $p$ (i.e., $\mathcal{Y} = \mathbb{R}^p$ and $\LY = \mathbb{R}^{p\times p}$ ).
\begin{definition}\label{def:separable_ovk} (Separable operator-valued kernel) \index{separable kernel} \\[0.1cm]
A separable operator-valued kernel is a function $K:\X\times\X \rightarrow \R^{p\times p}$, that can be written as 
\begin{equation}
\label{eq:sep_ovk}
K(\x,\z) = k(\x,\z)\mathbf{T}, \quad \forall \; \x, \z \in \X,
\end{equation} 
in which $k$ is a scalar-valued kernel function, and $\mathbf{T}\in \mathbb{R}^{p\times p}$ is a positive semi-definite matrix.
\end{definition}
This class of kernels is very attractive in terms of computational time, as it is easily decomposable. However the matrix $\mathbf{T}$ acts only on the outputs independently of the input data, which makes it difficult for these kernels to capture input-output relations. 
In the same spirit a more general class, sum of separable kernels, can be defined as follows.
\begin{definition}\label{def:sum_of_separable_ovk} (Sum of separable operator-valued kernels) \\[0.1cm]
A sum of separable operator-valued kernels is a function $K:\X\times\X \rightarrow \R^{p\times p}$, that can be written as 
\begin{equation}
\label{eq:ssep_ovk}
K(\x, \z) = \sum_{l} k_l(\x, \z)\mathbf{T}_l, \quad \forall\; \x, \z \in \X,
\end{equation}  
in which $k_l$ are a scalar-valued kernels and $\mathbf{T}_l \in \mathbb{R}^{p\times p}$ are positive semi-definite.
\end{definition}
This class of operator-valued kernels can capture more complex similarities, but still assumes that the unknown input-output dependencies can be decomposed into a product of two separate kernel functions that encode interactions
among inputs and outputs independently.

\begin{definition}\label{def:transformable_ovk} (Transformable operator-valued kernel) \index{transformable kernel} \\[0.1cm]
A transformable operator-valued kernel is a function $K:\X\times\X \rightarrow \R^{p\times p}$, that can be written as 
\begin{equation}
\label{eq:transformable_ovk}
K(\x, \z) = \left[\widetilde{k}(S_l \x, S_m \z)\right]_{l,m=1}^p, \quad \forall\; \x,\z \in \X,
\end{equation} 
in which $\widetilde{k}: \mathcal{\widetilde{X}} \times \mathcal{\widetilde{X}} \to \mathbb{R}$ is a scalar-valued kernel function and $\{S_t\}_{t=1}^p$ are mappings from $\mathcal{X}$ to $\mathcal{\widetilde{X}} $.
\end{definition}
In transformable kernels the data is transformed with the mappings $\{S_t\}_{t=1}^p$ before feeding it to the scalar-valued kernel function; which transformations to use depends on which outputs the element in $K(\x, \z)$ corresponds to. The mappings $S_t$ operate on input data while depending on outputs; however they are not intuitive nor easy to interpret and determine. 
One example of such kernels, which was proposed in ~\citet{Caponnetto2008universal}, is the kernel function $K(\x, \z):=(e^{\bm{\sigma}_{lm}\langle \x, \z \rangle}:l,m\in \{1,\ldots,p\})$ with $\bm{\sigma} = (\bm{\sigma}_{lm})$ a positive semi-definite matrix. In this transformable kernel the matrix entries outputted by the kernel are computed using linear transformations of the data. Indeed, it is easy to see that $K(\x,\z) = \big[\prod_{i=1}^p e^{\langle S_l^{(i)}\x, S_m^{(i)}\z\rangle}\big]_{l,m=1}^p$, where $\bm{\sigma} = \sum_{i=1}^p\lambda_i {\mathbf{u}_i} {\mathbf{u}_i} ^\top$ is the eigenvalue decomposition of $\bm{\sigma}$ and $S_t^{(i)}\x:=\sqrt{\lambda} \mathbf{u}_{it}\x$, for all $t=1,\ldots,p$.

Separable kernels are the most common operator-valued kernels to be used and learned. As already mentioned, they are nevertheless a relatively restrictive class of kernels, as the relationships between inputs and outputs are modelled independently of each other as was already illustrated in Figure~\ref{fig:separableIllustration1}. Moreover, only certain types of interactions can be modelled, as $\mathbf{T}$ should be a psd matrix, and symmetric. Figure~\ref{fig:sep_nonsep} illustrates this.

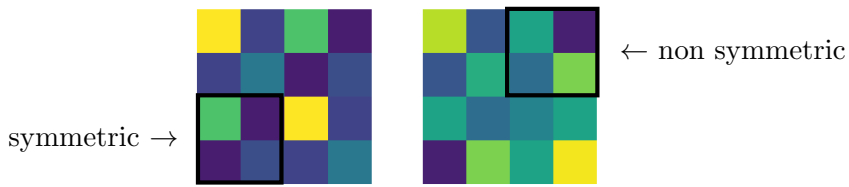
\begin{figure}[tb]
\centering
\begin{tikzpicture}
\node at (0,0) {\includegraphics[width=3cm]{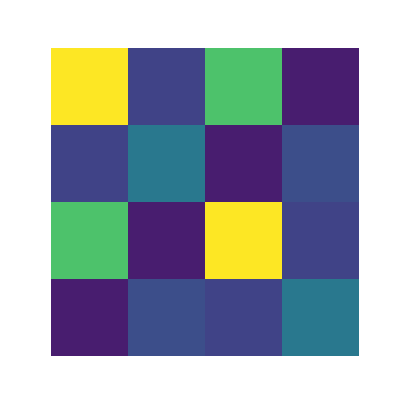}};
\draw [ultra thick] (-1.11, -1.16) rectangle (0,0);
\node at (-2.5,0.57-1.15) {symmetric $\rightarrow$};
\node at (3,0) {\includegraphics[width=3cm]{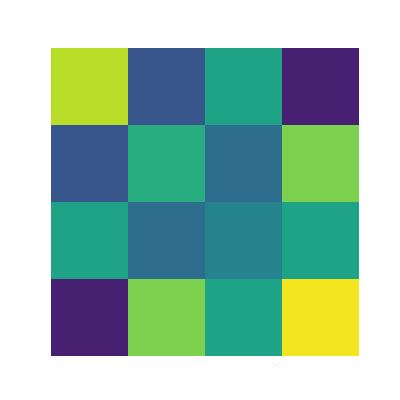}};
\draw [ultra thick] (3, 0) rectangle (4.2,1.15);
\node at (6,0.57) {$\leftarrow$ non symmetric};
\end{tikzpicture}
\caption{Illustration on differences of separable (left) and non-separable (right) operator-valued kernels. For separable kernels the $p\times p$ output matrix is always a psd symmetric matrix. }\label{fig:sep_nonsep}
\end{figure}

\subsection{Learning Separable Operator-valued Kernels}

Most works on learning operator-valued kernels consider the Output Kernel Learning (OKL) framework, where a separable operator-valued kernel is learned by fixing the scalar-valued kernel and learning the operator $\T$. The method is named for the observation that learning $\T$ does not depend on input data values, but only on the outputs.

All the output kernel learning algorithms are based on joint optimization, that is, the kernel is learned jointly with the learning problem, giving an optimization problem that generally can be written as
\[\min_{\mathbf{T},\mathbf{c}} \sum_{i=1}^n V(\mathbf{y}_i, f_{\mathbf{T},\mathbf{c}}(\mathbf{x}_i))+\lambda\,\Omega(f_{\mathbf{T},\mathbf{c}}) + \gamma\,\Theta(\T).\]
Here $V$ is the loss function for classification/regression and $\Omega$ is the accompanying regularization term, while $\Theta$ regularizes the output matrix. With separable operator-valued kernels, applying the representer theorem we get that $f_{T,\mathbf{c}}(\cdot) = \sum_{j=1}^n k(\x_j, \cdot)\T \mathbf{c}_j$.  

Many algorithms solve the output kernel learning problem. 
The first output-kernel learning algorithm was introduced in \citet{dinuzzo2011learning} with Frobenius norm regularizer on the output matrix $\T$. 
\citet{dinuzzo2011learning-b} considers learning low-rank output kernels, that is, separable kernels where the rank of $\T$ is constrained to be less or equal to some $r$. The optimization is performed with having also a regularizer on $\mathrm{tr}(\T)$ in addition to the rank constraint. More general or efficient formulations of output kernel learning have been proposed in~\citet{Ciliberto2015convex} and~\citet{jawanpuria2015efficient}.

To go beyond the standard OKL, 
\citet{kadri2012multiple} extended the multiple kernel learning framework~\citep{gonen2011multiple} that is popular in learning scalar-valued kernels into operator-valued kernel framework. The multiple kernel learning refers to paradigm where given multiple (scalar-valued) kernels $k_i$, a combination $k(\x, \z) = \sum_{i=1}^l \alpha_i k_i(\x, \z)$ is learned and then used in the predictive learning problem at hand. Similarly, \citet{kadri2012multiple} focus on learning a finite linear combination of separable operator-valued kernels.
They consider two formulations of the optimization problem. In the first one the separable operator-valued kernels all share the same output operator $\mathbf{T}$, meaning that the full kernel is \[K(\x, \z) = \sum_{i=1}^K \alpha_i k_i(\x, \z)\mathbf{T}.\] The second formulation considers the case where also the output operators differ across the operator-valued kernels in the sum, giving \[K(\x, \z) = \sum_{i=1}^K \alpha_i k_i(\x, \z)\mathbf{T}_i.\]  
In both of these versions only the multipliers $\alpha_i$ are learned, and the kernels are fixed, comparably to the case of multiple kernel learning (MKL) for scalar-valued kernels. Notably, the operators $\mathbf{T}$ and $\mathbf{T}_i$ are fixed in advance, which makes the use of this method difficult as it is not obvious how one should choose $\mathbf{T}$ without learning it.

Some works have continued this line of investigation.
\citet{sindhwani2013scalable} considers learning a combination of separable kernels that share the operator $\mathbf{T}$, while optimizing both the combination of the basis scalar-valued kernels and the matrix $\mathbf{T}$.
Another approach by \citet{gregorova2017forecasting} considers combining a set of scalar-valued kernels with a set of output matrices $\mathbf{T}_i$. However they impose diagonal structure on the output matrices, restricting the types of relations they are able to model. In this setting the diagonal values can be interpreted as model weights of the kernel in standard MKL setting.

\subsection{Learning Non-separable Operator-valued Kernels}

There are very few works that consider learning non-separable kernels. 
\citet{lim2015operator} considers an application to modelling time series data and goes further than separability by learning a combination of a separable and a transformable kernel. They consider a transformable kernel defined as \[[K_{transf.}(\x, \z)]_{st} = \exp(-\gamma (\x_s-\z_t)^2),\] where $\x_s$ and $\z_t$ are the $s$th and $t$th elements of vectors $\x$ and $\z$ respectively. This transformable kernel applies a Gaussian kernel to pairs of elements of the data vectors, giving a $d\times d$-matrix as an output if the data dimension is $d$. The separable kernel in their work is \[K_{sep.}(\x, \z) = \exp(-\gamma \|\x-\z\|^2)\, \mathbf{T},\] and the full kernel matrix they consider in learning is $K =K_{transf.} \circ K_{sep.}$, a Hadamard or element-wise matrix product of the two operator-valued kernel matrices. When they learn this kernel, they consider learning the matrix $\mathbf{T}$ from the separable part of it. 
This class of kernels cannot generalize to the learning problems we consider. The greatest restriction is, that the kernel outputs a $d\times d$ matrix, $d$ being the dimension of the input data. This is very rarely the same as the dimension of the outputs.

Another specialized operator-valued kernel is that of \citet{huusari2018mvml}, where a non-separable kernel is learned in context of multi-view learning, by incorporating a learnable metric operating between the views into the kernel. The output of a kernel is a $v\times v$ matrix where $v$ is the number of views in the data. Similarly to the previous work, this is not applicable for a general multi-output setting we consider. 
Having only few specialized works outside the separability framework motivates our more general entangled kernel learning paradigm.

\section{Partial Trace and Entangled Kernels}
\label{sec:ptk}

This section first revisits the known classes of operator-valued kernels and discusses the inclusions between them. After that we introduce the two novel classes of operator-valued kernels, the partial trace kernels that encompass the known classes of operator-valued kernels, and the entangled kernels that are a class of kernels distinct from the separable. 

While it is straightforward to see that separable kernels belong to the larger class of sum of separable, the picture is less clear for transformable kernels. The following examples clarify this situation. 

\begin{example} (transformable but not separable kernel)  \\[0.1cm]
	On the space $\mathcal{X} = \mathbb{R}$, consider the kernel
	\[K(x, z) = \begin{pmatrix}
	xz & xz^2 \\
	x^2z & x^2z^2
	\end{pmatrix}, \quad \forall\; x, z \in \X.\]
	$K$ is a transformable kernel, but not a (sum of) separable kernel. We obtain that $K$ is transformable simply by choosing in~Def.~\ref{def:transformable_ovk} the kernel  $\widetilde{k}(x, z) = xz$, $S_1(x)=x$, and $S_2(x)=x^2$. From the property of positive semi-definiteness of the operator-valued kernel, it is easy to see that the matrix $\mathbf{T}$ of a separable kernel  is symmetric~(see Def.~\ref{def:separable_ovk}), and since the matrix $K(x, z)$ is not, $K$ is not a separable kernel.
\end{example}

\begin{example}  (transformable and separable kernel)  \\[0.1cm]
	Let $K$ be the kernel function defined as
	\[K(\x, \z) = \langle \x, \z \rangle \mathbf{T}, \quad \forall\; \x, \z \in \X,\]
	where $\mathbf{T}\in\mathbb{R}^{p\times p}$ is a rank one positive semi-definite matrix. $K$ is both separable and transformable kernel. Since $\mathbf{T}$ is of rank one,  it follows that $\big(K(\x, \z)\big)_{lm} =  \mathbf{u}_l \mathbf{u}_m \langle \x, \z \rangle $, with $\mathbf{T} = \mathbf{u}  \mathbf{u}^\top$. We can see that $K$ is  transformable by replacing in~Def.~\ref{def:transformable_ovk} the kernel $\widetilde{k}(\x, \z)$ by $\langle \x, \z\rangle$ and $S_t(\x)$ by $\mathbf{u}_t\x$, $t=1,\ldots,p$.  $K$ is separable by construction.
\end{example}
It is worth noting that separable kernels are not limited to finite-dimensional output spaces, while transformable kernels are. Figure~\ref{fig:kc} depicts inclusions among kernel classes discussed here and the two new families of operator-valued kernels we propose: \textit{partial trace} kernels and \textit{entangled} kernels.

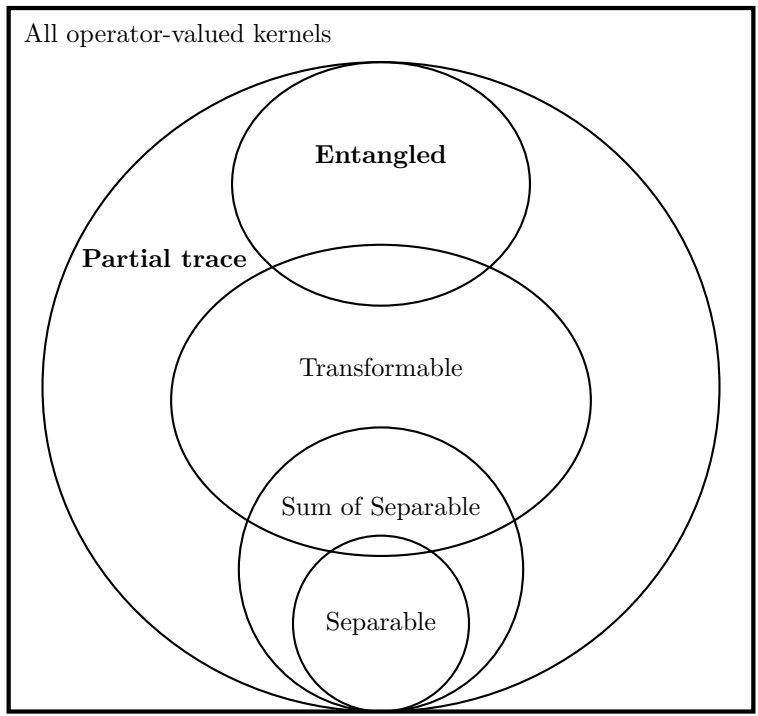
\begin{figure}[tb]
\centering

\begin{tikzpicture}[thick,scale=0.9, every node/.style={scale=0.9}]

\draw [ultra thick] (-5.5, -5.9) rectangle (5.5,4.5);
\node at (-3,4.1) {All operator-valued kernels};

\draw[thick] (0,-1.1) ellipse (5cm and 4.8cm);
\node at (-3.2,0.8) {\textbf{Partial trace}};

\draw[thick] (0,-3.8) ellipse (2.1cm and 2.1cm);
\node at (0,-2.9) {Sum of Separable};

\draw[thick] (0,-4.6) ellipse (1.3cm and 1.3cm);
\node at (0,-4.6) {Separable};

\draw[thick] (0,-1.3) ellipse (3.1cm and 2.3cm);
\node at (0,-0.8) {Transformable};

\draw[thick] (0,1.9) ellipse (2.2cm and 1.8cm);
\node at (0,2.3) {\textbf{Entangled}};

\end{tikzpicture}

\caption{Illustration of inclusions among various operator-valued kernel classes.
}   \label{fig:kc}

\end{figure}

We now define the two novel classes of operator-valued kernels. The first one, the class of partial trace kernels, encompasses both (sum of) separable and transformable kernels, while the second, entangled kernels, is a class of non-separable kernels. 
We start by introducing the more general class of partial trace kernels. 
The intuition behind this class of kernels is that in the scalar-valued case any kernel function $k$ can be written as the trace of an operator in $\mathcal{L(K)}$, where $\mathcal{K}$ is the reproducing kernel Hilbert space associated to the scalar-valued kernel $k$. It is easy to see that $k(\x, \z) = \langle \phi(\x), \phi(\z) \rangle = \mathrm{tr}(\phi(\x) \phi(\z)^\top)$.\footnote{\label{footnote:transpose}There is some abuse of notation in using the transpose symbol $^\top$ for a feature map representation which can be infinite-dimensional. In this case, we can write $k(\x, \z) = \mathrm{tr}(\phi(\x) \phi(\z)^*)$, where $\phi(\x)  \phi(\z)^*$ is the rank one operator defined for all $ u \in\mathcal{K}$ by $(\phi(\x)  \phi(\z)^*) u = \langle \phi(\z), u \rangle\phi(\x)$.}
The following definition of partial trace kernels can be motivated as a generalization of the kernel trick by using the partial trace operator instead of trace.
\begin{definition}\label{def:ptrKernel}(Partial trace kernel) \\[0.1cm] 
A partial trace kernel is an operator-valued kernel function $K$ having the following form
\begin{equation}
\label{eq:ptk}
K(\x, \z) = \mathrm{tr}_{\mathcal{K}}(\mathbf{P}_{\phi(\x),\phi(\z)}) ,
\end{equation}
where $\mathbf{P}_{\phi(\x),\phi(\z)}$ is an operator on $\mathcal{L(Y \otimes K)}$, and $\mathrm{tr}_{\mathcal{K}}$ is the partial trace on $\mathcal{K}$ (i.e., over the inputs).
\end{definition}
Depending on the choice of the operator $\mathbf{P}_{\phi(\x),\phi(\z)}$, partial trace kernels may not be positive semi-definite. Since the partial trace preserves positive semi-definiteness~\citep{filipiak2018properties}, it is clear that the partial trace kernel is positive semi-definite if for every $n\in\mathbb{N}$ and $\x_1, \ldots, \x_n\in \mathcal{X}$ the matrix $[\mathbf{P}_{\phi(\x_i),\phi(\x_j)}]_{i,j=1}^n$ is positive.
Even though all the kernel subclasses considered here and shown in Figure~\ref{fig:kc} are positive semi-definite, our definition of partial trace kernel is general and flexible enough to cover many kernels and allows the design of new ones that may or not be positive semi-definite.
As for the scalar-valued case, operator-valued kernels which are not positive semi-definite may be useful for learning in reproducing kernel Kre\u{\i}n spaces~\citep{ong2004learning,saha2020learning}.

The class of partial trace kernels  is very broad and encompasses the classes of separable and transformable kernels~(see Figure~\ref{fig:kc}). From the definition of the partial trace operation, we can see that if we choose $\mathbf{P}_{\phi(\x),\phi(\z)} = \sum_l \mathbf{T}_l \otimes \big(\phi_l(\x) \phi_l(\z)^*\big)$, where $\phi(\x)  \phi(\z)^*$ is the rank one operator defined in Footnote~\ref{footnote:transpose}, we recover the case of sum of separable kernels. In the same way, if we fix $[\mathbf{P}_{\widetilde{\phi}(\x),\widetilde{\phi}(\z)}]_{l,m=1}^p = \big(\widetilde{\phi} \circ S_l(\x)\big) \big(\widetilde{\phi} \circ S_m(\z)\big)^\top$ in~ Eq.~\ref{eq:ptk}, computing the trace of each block using the partial trace will give the transformable kernel.
With this in mind, we can use the partial trace kernel formulation to induce a novel class for operator-valued kernels which are not separable, with the goal to characterize inseparable correlations between inputs and outputs. 
\begin{definition}\label{de:entangled}(Entangled kernel) \\[0.1cm] 
	An entangled  operator-valued kernel $K$ is defined as 
	\begin{equation}
	\label{eq:ek}
	K(\x, \z) = \mathrm{tr}_{\mathcal{K}}\left(\mathbf{U} \big(\mathbf{T} \otimes (\phi(\x)  \phi(\z)^*)\big)\mathbf{U}^*\right),
	\end{equation}
	in which $\mathbf{T}\in\mathcal{L(Y)}$ is a positive semi-definite operator, and $\mathbf{U} \in \mathcal{L(Y\otimes K)}$ is not separable.
\end{definition}
When $\mathcal{Y}$ and $\mathcal{K}$ have finite dimensions $p$ and $N$, respectively, $\mathcal{L(Y\otimes K)}$ is simply the set of matrices of dimensions $pN\times pN$ and $\phi(\x)  \phi(\z)^*$ is the matrix $\phi(\x)  \phi(\z)^\top$, where  $\phi(\z)^\top$ denotes the transpose of  $\phi(\z)$.
In the following we abuse the notation and denote  $N$ as the dimensionality of feature representation $\phi(\x)$. 
However we do not restrict ourselves to finite dimensions and $N$ in this notation can also be infinite.
In this definition, $\mathbf{U}$ not being separable means that it cannot be written as $\mathbf{U} = \mathbf{B} \otimes \mathbf{C}$, with $\mathbf{B}\in\mathbb{R}^{p\times p}$ and $\mathbf{C}\in\mathbb{R}^{N\times N}$. The term $\mathbf{T} \otimes (\phi(\x) \phi(\z)^\top)$ represents a separable kernel function over inputs and outputs, while $\mathbf{U}$ characterizes the entanglement shared between them. 
We note that $\mathbf{U}$ not being equal to  $\mathbf{B} \otimes \mathbf{C}$ marks the crucial difference to separable kernels; if $\mathbf{U}$ were  $\mathbf{B} \otimes \mathbf{C}$, then the class described above would be part of separable kernels (see Theorem~\ref{th:ptr_separable}).

Some intuition to $\mathbf{U}$ can be seen from its role of an \enquote{entangled} similarity in the joint feature space.
It is entangled in the sense that it cannot be decomposed into two \enquote{sub}-matrices of similarity between inputs and between outputs independently. The partial trace is the operation used to recover the sub-similarity matrix between the outputs from the entangled joint similarity matrix. In the particular case of separability, the partial trace will give the output metric.

\begin{theorem}\label{th:ptKernelsArePSD}  Entangled kernels given by the Definition \ref{de:entangled} are positive semi-definite kernels.
\end{theorem}
\begin{proof}
  	The proof is based on the observation that $\mathbf{T} \otimes (\phi(\x)  \phi(\x)^\top) = \mathbf{O}_{\phi(\x)}\mathbf{T}\mathbf{O}_{\phi(\x)}^*$, where $\mathbf{O}_{\phi(\x)}$ is the operator defined by
  	\begin{align*}
  		\mathbf{O}_{\phi(\x)} : \mathcal{Y} &\longrightarrow \mathcal{Y} \otimes \mathcal{K}\\
  		 \mathbf{y} &\longmapsto \mathbf{y} \otimes \phi(\x),
  	\end{align*}
  	and $\mathbf{O}_{\phi(\x)}^*$ is its adjoint defined by:
  	\begin{align*}
  	\mathbf{O}_{\phi(\x)}^* :  \mathcal{Y} \otimes \mathcal{K} &\longrightarrow \mathcal{Y} \\
  	\mathbf{y} \otimes u  &\longmapsto \langle \phi(\x), u \rangle \mathbf{y}.
  	\end{align*}
  	Indeed, $\forall \, \mathbf{y}\in\mathcal{Y}, u\in\mathcal{K}$, we have 
\begin{align*}
\big(\mathbf{T} \otimes (\phi(\x)  \phi(\z)^\top) \big)(\mathbf{y} \otimes u) &= \mathbf{Ty} \otimes \phi(\x)  \phi(\z)^\top u \\ 
&= \langle \phi(\z), u \rangle \mathbf{Ty} \otimes \phi(\x) \\ 
&= \mathbf{O}_{\phi(\x)}  \mathbf{T} \langle \phi(\z), u \rangle \mathbf{y}  \\
&= \mathbf{O}_{\phi(\x)}\mathbf{T}\mathbf{O}_{\phi(\z)}^* (\mathbf{y} \otimes u).
\end{align*}  
  	Now, to show that the entangled kernel defined by Eq.~\ref{eq:ek} is positive semi-definite, let us compute 
\begin{align*}
	\sum_{i,j} \langle
\mathbf{y}_i, K(\x_{i},\x_{j})\mathbf{y}_{j}\rangle_{\mathcal{Y}} &= \sum_{i,j} \left\langle
	\mathbf{y}_i, \mathrm{tr}_{\mathcal{K}}\left(\mathbf{U} \big(\mathbf{T} \otimes (\phi(\x_i)  \phi(\x_j)^\top)\big)\mathbf{U}^\top\right) \mathbf{y}_{j}\right\rangle_{\mathcal{Y}}  \\
	&=\sum_{i,j} \mathrm{tr} \left(\big(\mathbf{y}_i^\top \otimes \mathbf{I}_N\big) \mathbf{U} \big(\mathbf{T} \otimes (\phi(\x_i)  \phi(\x_j)^\top)\big)\mathbf{U}^\top\big(\mathbf{y}_j \otimes \mathbf{I}_N\big)\right) \quad \footnotemark  \\
	& = \sum_{i,j} \mathrm{tr} \left(\big(\mathbf{y}_i^\top \otimes \mathbf{I}_N\big) \mathbf{U} \big(O_{\phi(\x_i)}\mathbf{T}O_{\phi(\x_j)}^* \big)\mathbf{U}^\top\big(\mathbf{y}_j \otimes \mathbf{I}_N\big)\right)\\
	& =  \mathrm{tr}(S\mathbf{T}S^*) \geq 0,
\end{align*}   
as clearly $\mathrm{tr}(S\mathbf{T}S^*) =  \mathrm{tr}(S\mathbf{V}\mathbf{V}^\top S^*) = \mathrm{tr}((S\mathbf{V}(S\mathbf{V})^*) $ and trace preserves positivity. Here we have denoted $S = \sum_i \big(\mathbf{y}_i^\top \otimes \mathbf{I}_N\big) \mathbf{U} O_{\phi(\x_i)}$. 
Moreover, from the fact that $\mathrm{tr}_{\mathcal{K}}(\mathbf{A}^\top) = \mathrm{tr}_{\mathcal{K}}(\mathbf{A})^\top$, we immediately obtain that $K(\x, \z) = K(\z,\x)^*$, which completes the proof. \footnotetext{See~\citet[Lemma 2.11]{filipiak2018properties}.}
  
\end{proof}

While by definition entangled kernels cannot be separable, the following example shows that an entangled kernel can be transformable.
\begin{example}  (Entangled and transformable kernel)  \\[0.1cm]
An entangled kernel given by Definition~\ref{de:entangled} with symmetric and supersymmetric $\mathbf{U}$ (that is, all its $N\times N$ blocks are symmetric), $\mathbf{T} = \mathbbm{1}_p\mathbbm{1}_p^T \in\R^{p\times p}$ and with a scalar-valued kernel with finite-dimensional feature mapping, $\phi(\x)\in\R^N$, is a transformable kernel. 

This can be seen from the following calculation:
\begin{align*}
\mathbf{P}_{\phi(\x),\phi(\z)} &= \left[ \sum_{i,j=1}^p \mathbf{U}_{lj} \phi(\x)\phi(\z)^\top \mathbf{U}_{im}^\top \right]_{l,m=1}^p \\
&= \left[ \left( \sum_{j=1}^p \mathbf{U}_{lj} \phi(\x) \right) \left( \sum_{i=1}^p  \mathbf{U}_{im}\phi(\z) \right)^\top  \right]_{l,m=1}^p \\
&= \left[ \left( \sum_{j=1}^p \mathbf{U}_{lj} \phi(\x) \right) \left( \sum_{i=1}^p  \mathbf{U}_{mi}\phi(\z) \right)^\top  \right]_{l,m=1}^p.
\end{align*}
Now $\mathbf{S}_l(\x) =  \sum_{k=1}^p \mathbf{U}_{lk} \phi(\x)$; recall that transformable kernels can be obtained from partial trace kernels (\ref{eq:ptk}) with $[\mathbf{P}_{\widetilde{\phi}(\x),\widetilde{\phi}(\z)}]_{l,m=1}^p = \big(\widetilde{\phi} \circ S_l(\x)\big) \big(\widetilde{\phi} \circ S_m(\z)\big)^\top$. 
\end{example}

Choice of the matrix $\mathbf{U}$ is crucial to the class of entangled kernels. In the next section we develop an algorithm that learns  an entangled kernel from data.

\section{Entangled Kernel Learning}
\label{sec:ekl}

 In general,  there is no knowing whether input and output data are or are not entangled. In this sense, learning  the entangled $K$ in Eq.~\ref{eq:ek} by imposing that~$\mathbf{U}$ is inseparable can sometimes be restrictive. In our entangled kernel learning approach we do not impose any separability restriction, with the hope that our learning algorithm can automatically detect the lack or presence of entanglement.
Key to our method is a reformulation of the entangled kernel $K$~(Eq.~\ref{eq:ek}) via Choi-Kraus representation. 

\begin{theorem}\label{th:choiKraus}(Choi-Kraus representation~\citealp{choi1975completely,kraus1983states,rieffel2011quantum}) \\[0.1cm]
The map $	K(\x, \z) = \mathrm{tr}_{\mathcal{K}}\left(\mathbf{U} \big(\mathbf{T} \otimes (\phi(\x) \phi(\z)^\top)\big)\mathbf{U}^\top\right)$ can be generated by an operator sum representation  
\begin{equation}
\label{eq:ek-kraus}
K(\x, \z) =  \sum_{i=1}^r \mathbf{M}_i \phi(\x) \phi(\z)^\top \mathbf{M}_i^\top,
\end{equation}
where $r$ is called the \emph{Kraus rank} and $\mathbf{M}_i\in \mathbb{R}^{p\times N}$ are the \emph{Kraus operators}. 
\end{theorem}
\begin{proof}
	See Appendix~\ref{app:choiKraus}
\end{proof}
Using this formulation, entangled kernel learning consists of finding a (possibly low-rank when $r$ is small) decomposition of the kernel by learning the matrices $\mathbf{M}_i$, $i=1,\ldots,r$, 
where these matrices \enquote{merge} the matrices $\mathbf{T}$ and $\mathbf{U}$. 

Looking at the simpler class of separable kernels, they model the relations between tasks~(i.e., ouputs)  using only the output matrix~$\mathbf{T}$~of size $p\times p$. The matrix~$\mathbf{T}$ acts as a covariance matrix on the labels independently of the inputs and if connected to the existing deep neural network approaches for multi-task learning (i.e.  the task-decoder methods which learn a shared representation at a high-level layer followed by task-specific decoders, see, e.g.,~\citealp{meyerson2018beyond}), can then play the role of a decoder that predicts labels using the outputs of related tasks. 
Entangled kernels, however, capture task relatedness through the matrices $\mathbf{M}_i$, of size $p\times N$ modeling the relationships of tasks and features. 
A small value of $r$ is a (Kraus) low-rank assumption which reduces the number of parameters and promotes sharing information and knowledge between tasks. 
This would be more similar to the column-based approaches in deep multi-task learning, which consider multiple deep neural networks in parallel and share the layer parameters between these networks~\citep{meyerson2018beyond}.

It is not easy to see from the theorem how exactly $\mathbf{T}$, $\mathbf{U}$ and $\mathbf{M}_i$ interact. While the proof of the theorem (see Appendix~\ref{app:choiKraus}) gives the explicit relation between them, in order to make this more clear let us consider only one element of the output of the kernel. From the Choi-Kraus representation~(\ref{eq:ek-kraus}) we have
\[ [K(\x, \z)]_{s,t} = \sum_{i=1}^r \mathbf{M}_i[s,:] \phi(\x) \phi(\z)^\top (\mathbf{M}_i[t,:])^\top, \]
and on the other hand from the definition of the entangled kernels~(Definition~\ref{de:entangled}) we get that
\[[K(\x, \z)]_{s,t} = tr\left( \sum_{l=1}^p\sum_{k=1}^p \mathbf{U}_{sk} \mathbf{T}_{kl} \phi(\x) \phi(\z)^\top  \mathbf{U}_{lt}^\top \right) = \sum_{l,k=1}^p \mathbf{T}_{kl} \, \phi(\z)^\top \mathbf{U}_{lt}^\top \mathbf{U}_{sk}\phi(\x). \]

Here we have used notation $\mathbf{M}_i[s,:]$ to refer to the row $s$ of matrix $\mathbf{M}_i$, $\mathbf{T}_{ij}$ to refer to the element in position $(i,j)$ in $\mathbf{T}$, and $\mathbf{U}_{ij}$ similarly ordered to the block of size $N\times N$ in $\mathbf{U}$. 
From this we see that both the rows of $\mathbf{M}_i$ and rows of (blocks of) $\mathbf{U}$ act on transforming the features $\phi(\x)$. If we restrict the \enquote{rank} of the entangled kernel by restricting the $r$ in~(\ref{eq:ek-kraus}), we are in essence restricting the row space of (blocks of) $\mathbf{U}$.

It is important to note, that while every entangled kernel can be represented like this, the representation is not restricted only to entangled kernels. Thus by learning the $\mathbf{M}_i$ we expect to learn the meaningful relationships in the data, be they entangled or not. Yet, when learning even a low-rank entangled kernel, we are bound to learn more parameters in the $\mathbf{M}_i$ than we could learn from just (full-rank) $\mathbf{T}$, thus making the class of kernels we consider much more expressive.

To make this explicit, let us consider the separable kernels in the Choi-Kraus represenation framework. For the class of separable kernels $\mathbf{U}$ equals identity, and the $\mathbf{T}$ in~(\ref{eq:ek}) is the same as the $\mathbf{T}$ in Definition~\ref{def:separable_ovk}. We can describe the operator $\mathbf{T}$ with a set of vectors $\mathbf{t}_j \in \R^p$ for which $\mathbf{T} = \sum_j \mathbf{t}_j \mathbf{t}_j^\top$. 
\begin{remark}\label{rm:ck_sep} (Choi-Kraus representation of seprabale kernels)  \\[0.1cm]
For separable kernels, each $\mathbf{M}_i$ in the Choi-Kraus represenation (\ref{eq:ek-kraus}) can be described as $\mathbf{M}_i = \mathbf{M}_{jk} = \mathbf{t}_j \mathbf{e}_k^\top$,  with $j=1,...,p$, $k=1,...,N$ and $\{\mathbf{e}_k\}_{k}$ is an orthonormal basis of $\R^{N}$.
\end{remark}
This can be confirmed with the following calculation: 
	\begin{align*}
	K(\x, \z) &= \sum_{i=1}^r \mathbf{M}_i \phi(\x) \phi(\z)^\top \mathbf{M}_i^\top =
	\sum_{j=1}^p \sum_{k=1}^N  \mathbf{t}_j \mathbf{e}_k^\top    \phi(\x) \phi(\z)^\top   \mathbf{e}_k \mathbf{t}_j^\top \\
	&=  \sum_{k=1}^N  \mathbf{e}_k^\top    \phi(\x) \phi(\z)^\top   \mathbf{e}_k \sum_{j=1}^p   \mathbf{t}_j  \mathbf{t}_j^\top = \sum_{k=1}^N \langle  \phi(\x) \phi(\z)^\top   \mathbf{e}_k , \mathbf{e}_k \rangle \mathbf{T} \\
	&= \mathrm{tr}\left(\phi(\x) \phi(\z)^\top  \right) \, \mathbf{T} =  \langle  \phi(\x) , \phi(\z)\rangle \, \mathbf{T} =  k(\x, \z) \, \mathbf{T}. 
	\end{align*}
It is clear that the class of entangled kernels is much more expressive than the class of separable kernels. With this in mind, we now turn our attention to describing the framework in which we can learn the $\mathbf{M}_i$ in entangled kernels.

Because the feature space of the scalar-valued kernel can be of very large dimensionality~(or infinite-dimensional), we consider an approximation to speed up the computation. For example random Fourier features or Nyström approximation~\citep{rahimi2008random,williams2001using} give us $\hat{\phi}$ such that $$k(\x, \z) = \langle \phi(\x), \phi(\z) \rangle \approx\langle \hat{\phi}(\x), \hat{\phi}(\z) \rangle.$$
We note that our approximation is on scalar-valued kernels, not operator-valued, although there are methods for approximating them, too, directly~\citep{brault2016random,minh2016operator}.
Our approximated kernel is thus 
 \begin{equation}
\label{eq:Kapprox}
\hat{K}(\x, \z) =  \sum_{i=1}^r \hat{\mathbf{M}}_i \hat{\phi}(\x) \hat{\phi}(\z)^\top \hat{\mathbf{M}}_i^\top,
\end{equation}
where $\hat{\phi}(\x) \in \R^m$ and $\hat{\mathbf{M}}_i \in \R^{p\times m}$, from where our goal is to learn the $\hat{\mathbf{M}}_i$.
We can write our $np\times np$ kernel matrix as 
 \setcounter{footnote}{3}
  \addtocounter{footnote}{1}
\footnotetext{Matrix cookbook Eq.~(520): 
\url{http://www.math.uwaterloo.ca/~hwolkowi/matrixcookbook.pdf}.}
  \setcounter{footnote}{3}
\begin{align}
\hat{\mathbf{G}} &= \sum_{i=1}^r \vvec(\hat{\mathbf{M}}_i\hat{\bm{\Phi}})\vvec(\hat{\mathbf{M}}_i\hat{\bm{\Phi}})^\top  \label{eq:G}\\
&= \sum_{i=1}^r (\hat{\bm{\Phi}}^\top \otimes \mathbf{I}_p )\vvec(\hat{\mathbf{M}}_i)\vvec(\hat{\mathbf{M}}_i)^\top (\hat{\bm{\Phi}} \otimes \mathbf{I}_p )\;\;\;\footnotemark \notag
\end{align} 
where $\hat{\bm{\Phi}}= [\hat{\phi}(\x_1), \cdots, \hat{\phi}(\x_n)]$ is of size $m\times n$. 
Further, if we denote 
\begin{equation}
\D = \sum_{i=1}^r \vvec(\M_i)\vvec(\M_i)^\top,
\end{equation}
we can simply write
\begin{equation}\label{eq:eklKernel}
\hat{\mathbf{G}} = (\hat{\bm{\Phi}}^\top \otimes \mathbf{I}_p )\D (\hat{\bm{\Phi}} \otimes \mathbf{I}_p ).
\end{equation}
To learn an entangled kernel, we need to learn the psd matrix $\mathbf{D}$. 

It is good to note that also from this formulation of the learnable entangled kernel, we can easily recover the class of separable kernels. Indeed, as we already mentioned, the representation (\ref{eq:ek-kraus}) is not restricted to only entangled kernels, and thus it is possible to recover other frameworks also from (\ref{eq:eklKernel}). In this case, if we choose $\D=(\mathbf{I}_N \otimes \mathbf{T})$ with $\mathbf{T}\in\R^{p\times p}$ (note that requirement for $\D$ to be psd and symmetric also restricts $\mathbf{T}$ to be such), we can write \[\hat{\mathbf{G}} =(\hat{\bm{\Phi}}^\top \otimes \mathbf{I}_p ) (\mathbf{I}_N \otimes \mathbf{T}) (\hat{\bm{\Phi}} \otimes \mathbf{I}_p ) = (\hat{\bm{\Phi}}^\top\mathbf{I}_N \hat{\bm{\Phi}}  \otimes \mathbf{I}_p\mathbf{T}\mathbf{I}_p )  = \mathbf{K} \otimes \mathbf{T},\] and see that we have recovered a separable kernel. Notably, we can see that our proposed class of learnable kernels are much more expressive than that of separable kernels, as they can be recovered as a special case.

\subsection{The Learning Algorithm}

We will now describe how to learn the psd matrix $\mathbf{D}$ from the entangled kernel~(\ref{eq:eklKernel}), and how to efficiently formulate the vector-valued learning problem.

Kernel alignment \citep{cristianini2002kernel,cortes2012algorithms} is a measure of similarity between two (scalar-valued) kernels. 
Alignment between two matrices $\mathbf{M}$ and $\mathbf{N}$ is defined as \begin{equation}\label{def:alignment}
A(\mathbf{M}, \mathbf{N}) = \frac{\left\langle \mathbf{M}_c, \mathbf{N}_c \right\rangle_F}{\|\mathbf{M}_c\|_F\|\mathbf{N}_c\|_F},
\end{equation}
where subscript c refers to centered matrices, that is, $\mathbf{M}_c = \mathbf{HMH}$ where $\mathbf{H} = \mathbf{I}_n-\tfrac{1}{n}\bm{1}\bm{1}^\top$, if $\mathbf{M}$ is a $n \times n$ matrix. Here $0\leq A(\mathbf{M}, \mathbf{N}) \leq 1$, with higher value showing better alignment and higher similarity. 
Kernel alignment has been used in learning a kernel by considering the so-called ideal kernel as the target of the alignment. In the context of binary classification, the ideal kernel is defined as $\mathbf{yy}^\top$ where $\mathbf{y} =[y_1, ...,y_n]^\top$ is the $n$-vector of class labels~\citep{cristianini2002kernel}. 
The work of \citet{kandola2002extensions} extends the ideal kernel to regression and to unbalanced classification where there are more samples from one class than from the other.

We extend the concept of alignment into the case of multiple outputs.
As already mentioned, in the previous works the ideal kernel has been the linear kernel defined on output values, $\mathbf{yy}^\top$. Extended to multi-output setting, it is natural to consider the linear kernel $\mathbf{K}_Y = \mathbf{Y}^\top\mathbf{Y}$, where $\mathbf{Y}$ is of size $p\times n$, containing the labels associated to data sample $i$ on its $i$th column. 
Yet, if we wish to use this ideal kernel in learning our entangled kernel we face problems as we would be trying to align a $np\times np$ matrix with a $n\times n$ one. 
We thus propose as the first part of our optimization problem to align the partial trace of our entangled kernel matrix, $\mathrm{tr}_p(\hat{\mathbf{G}})$, to the output kernel $\mathbf{K}_Y$. 
As this term does not consider the full operator-valued kernel matrix, for the second term we consider $\mathbf{y} = \vvec(\mathbf{Y})$, a vectorization of the matrix containing the labels, and take matrix $\mathbf{K}_y = \mathbf{yy}^\top $ as the second ideal kernel in our setting. In this kernel each $p\times p$ block is an outer product between the labels associated with the samples, or $[\mathbf{K}_y]_{ij} = \mathbf{y}_i\mathbf{y}_j^\top$, and  we can directly align the $np \times np$  matrix $\mathbf{K}_y$ to the full entangled kernel $\hat{\mathbf{G}}$.
Our optimization problem is thus a convex combination 
\begin{equation}\label{eq:ekl_opt}
\max_\D \quad  (1-\gamma)\, A \left( \mathrm{tr}_p(\hat{\mathbf{G}}), \mathbf{Y^\top Y}\right) + \gamma \, A \left( \hat{\mathbf{G}}, \mathbf{y}\mathbf{y}^\top \right),
\end{equation} 
where $\gamma\in [0,1]$. We note that by applying Lemma 2.11 from \cite{filipiak2018properties}, we can write $\mathrm{tr}_p(\hat{\mathbf{G}}) = \hat{\bm{\Phi}}^\top  \mathrm{tr}_p(\D) \hat{\bm{\Phi}}$.

Intuitively the first alignment learns a scalar-valued kernel matrix that can be obtained via partial trace applied to the more complex operator-valued kernel, while the second term focuses on the possibly entangled relationships in the data. 
One possibility for using the entangled kernel framework is to learn a scalar-valued kernel  for multi-output problem using the partial-trace formulation and to use it in a kernel machine.

The optimization problem is solved with a gradient-based approach. To make sure that the resulting kernel is valid (psd),  we write $\D= \mathbf{Q}\mathbf{Q}^\top$ with $\mathbf{Q}$ of size $mp\times r$ with $r$ at most $mp$, and perform the optimization over $\mathbf{Q}$. The gradients for alignment terms are straightforward to calculate. The optimization is performed on sphere manifold as a way to normalize $\mathbf{D}$.\footnote{We used the toolbox from \url{pymanopt.github.io} for the implementation~\citep{JMLR:v17:16-177}.}

\begin{algorithm}[tb]
\caption{Entangled Kernel Learning (EKL)}\label{alg:ekl}
\begin{algorithmic}
\STATE {\textbf{Input:} matrix of features $\bf{\Phi}$, labels $\mathbf{Y}$}
\STATE {// 1) Kernel learning:  }
\STATE {Solve for $\mathbf{Q}$ in eq.\ref{eq:ekl_opt} ($\mathbf{D}=\mathbf{QQ}^\top$) within a sphere manifold}
\STATE {// 2) Learning the predictive function:}
\IF{Predict with scalar-valued kernel} 
\STATE {$\mathbf{c}_K = (\mathrm{tr}_p(\hat{\mathbf{G}}) +\lambda \mathbf{I})^{-1} \mathbf{Y}^\top $ \hfill $\mathcal{O}(m^3 + mnp)$  } \ELSE 
\STATE {$\mathbf{c}_G=(\hat{\mathbf{G}}+\lambda \mathbf{I})^{-1} \vvec(\mathbf{Y}) $ \hfill $\mathcal{O}(r^3 + mnp^2)$  } \ENDIF
\STATE {\textbf{Return} $\mathbf{D}=\mathbf{QQ}^\top$, $\mathbf{c}$}
\end{algorithmic}
\end{algorithm}
%

After we have learned the entangled kernel, we solve the learning problem by choosing the squared loss function
\begin{equation}\label{eq:EKL_opt1}
\min_{\mathbf{c}} \| \vvec(\mathbf{Y}) - \mathbf{\hat{G}c} \|^2 + \lambda \langle \mathbf{\hat{G}c}, \mathbf{c}\rangle.
\end{equation}
For this $\mathbf{c}$ update we can find the classical closed-form solution, $\mathbf{c} = (\hat{\mathbf{G}}+\lambda \mathbf{I})^{-1} \vvec(\mathbf{Y})$. 
Note that this computation is, by considering the entangled structure of $\hat{\mathbf{G}}$, more computationally efficient than a general (say, some transformable) operator-valued kernel. Generally we can say that the complexity of calculating the predictive function with nonseparable operator-valued kernels is $\mathcal{O}(n^3p^3)$. In our proposed network, however, we can apply the Woodbury formula for the matrix inversion and only invert a $r\times r$ matrix. 
We can assume that $m \ll n$ and often in multi-output data sets $p\ll n$. Furthermore as $r\leq mp$, we can see that $n$ dominates the complexity calculations. Thus we write the complexity of the $\mathbf{c}$-update as $\mathcal{O}(r^3+mnp^2)$, where we have also considered the most costly steps of the matrix multiplications involved.

Moreover it is possible, with our kernel learning framework, to learn a scalar-valued kernel by first learning the entangled kernel as proposed and then extracting the scalar-valued one by using partial trace operator. The scalar-valued kernel proposed to be used in predicting is now $\mathbf{K} = \mathrm{tr}_p(\hat{\mathbf{G}})$, and it can be used in regression setting as usual; calculating $\mathbf{c}_K = (\mathbf{K}+\lambda \mathbf{I})^{-1} \mathbf{Y}^\top $ and using that to obtain predictions. We note that even here our framework brings forward advancements in computational complexity; after having an entangled kernel the cost of calculating $\mathbf{c}_K$ using again Woodbury formula is centered around inverting a $m\times m $ matrix. 
Again when we assume $m\ll n$, and consider the most dominant term arising from the matrix multiplications, the complexity of the $\mathbf{c}_K$-step is $\mathcal{O}(m^3+mnp)$.

There are also gains in predictive complexity. Predicting with a general operator-valued kernel has the complexity $\OO(tnp^2)$. With an entangled kernel of our formulation, this reduces to $\OO((n+t+r)mp)$. For the ptrEKL kernel, assuming partial trace of $\mathbf{D}$ is already calculated, the predictive complexity is $\OO(tm^2+(n+t)mp)$ instead of $\OO(tnp)$, beneficial with small $m$. 
The complexities of calculating the parameters ($\mathbf{c}$) of the predictive function and predicting with various operator-valued kernels are summarized in Table~\ref{tb:complexities}. 
It is good to note that the complexities for separable kernels could be reduced by approximating the scalar-valued kernel matrix, as well as that it would be also possible to approximate a general operator-valued kernel matrix. 
While we show this approximation for separable kernels in Table~\ref{tb:complexities} for completeness, it is good to remember that these kind of approximations are not intrinsic to the approaches, unlike with our entangled kernels.
Additionally we have assumed in the complexities for general and separable kernels that the big operator-valued kernel matrix $\mathbf{G}$ and scalar-valued kernel matrix $\mathbf{K}$ as well as the output matrix $\mathbf{T}$ are already known at the time of the computations. Similarly, we have assumed that for entangled kernels the $\Phi$ and $\mathbf{Q}$ are already available at the time of these calculations.

\begin{table}
\centering
\begin{tabular}{lll}
\toprule
    & learning $\mathbf{c}$ & predicting $\mathbf{Y}$ \\
\midrule
No structure & $\OO(n^3p^3)$ & $\OO(tnp^2)$ \\
Separable & $\OO(n^3+p^3)$ & $\OO(tn p+np^2)$ \\
Low-rank separable & $\OO(m^3 + m^2n+p^3)$  & $\OO(tp^2+(n+t)pm)$ \\
Entangled & $\OO(r^3 + mnp^2)$ & $\OO((t+n+r)pm)$ \\
Entangled, ptr & $\OO(m^3 + mnp)$ & $\OO(tm^2+(n+t)pm)$ \\
\bottomrule
\end{tabular}
\caption{Complexities of calculating parameters of predictive function (vector $\mathbf{c}$), and predicting labels (calculating $\mathbf{G}_t\mathbf{c}$ with $\mathbf{G}_t$ of size $tp\times np$) with various operator-valued kernels. 
The \enquote{low-rank separable} refers to the case when the scalar-valued kernel matrix of the separable kernel has a low-rank approximation, i.e. $\mathbf{G} = \mathbf{K\otimes T} = \mathbf{UU^\top\otimes T}$ with some $\mathbf{U}$ of size $n\times m$ with $m<n$.
} \label{tb:complexities}
\end{table}

\subsection{Rademacher Generalization  Bound}

We now provide a generalization analysis of our EKL algorithm using Rademacher complexities~\citep{bartlett2002rademacher}. The notion of Rademacher complexity has been generalized to vector-valued hypothesis spaces~\citep{maurer2006rademacher,sindhwani2013scalable,sangnier2016joint}. 
Previous work has analyzed the case where the matrix-valued kernel is fixed prior to learning, while our analysis considers the kernel learning problem, similarly to the bound in~\citet{huusari2018mvml}. It provides a Rademacher bound for our algorithm when both the vector-valued function $f$ and the kernel via $\mathbf{Q}$, are learnt. 
We start by recalling that the feature map associated to the matrix-valued kernel $K$ is the mapping $\Gamma: \mathcal{X} \to \mathcal{L}(\mathcal{Y}, \mathcal{H}_K)$, where $\mathcal{X}$ is the input space, $\mathcal{Y}=\mathbb{R}^p$, and $\mathcal{L}(\mathcal{Y}, \mathcal{H}_K)$ is the set of bounded linear operators from $\mathcal{Y}$ to $\mathcal{H}_K$~(see, e.g., ~\citealp{Micchelli2005onlearning,Carmeli2010vector} for more details). It is known that $K(\x, \z) = \Gamma(\x)^*\Gamma(\z)$. We denote by $\Gamma_\mathbf{Q}$ the feature map associated to our entangled kernel~(Equation~\ref{eq:eklKernel}).

The hypothesis class of EKL is 
\begin{equation*}
\mathcal{H}_\beta = \{ x\mapsto f_{u,\mathbf{Q}}(\x) = \Gamma_\mathbf{Q}(\x)^* u : \mathbf{Q}\in \Delta, \|u\|_{\mathcal{H}} \leq \beta\},
\end{equation*}
with $ \Delta = \{\mathbf{Q}: \,\|\mathbf{Q}\|_F = 1 \}$ and $\beta$ is a regularization parameter. 
Let $\boldsymbol{\sigma}_1,\ldots,\boldsymbol{\sigma}_v$ be an iid family of vectors of independent Rademacher variables where $\boldsymbol{\sigma}_i\in \mathbb{R}^p, \; \forall\, i=1,\ldots,n$. The empirical Rademacher complexity of the vector-valued class $\mathcal{H}_\beta$ is the function $\hat{\mathcal{R}}_n( \mathcal{H}_\beta)$ defined as
\begin{equation*}
\hat{\mathcal{R}}_n( \mathcal{H}_\beta) = \frac{1}{n} \E\left[\sup_{f\in\mathcal{H}} \sup_{\mathbf{Q}\in \Delta}\sum_{i=1}^n \boldsymbol{\sigma}_i^\top f_{u,\mathbf{Q}}(\x_i)\right].
\end{equation*}

\begin{theorem}\label{th:EKL_Rademacher} (Rademacher complexity bound for EKL)  \\ [0.1cm]
The empirical Rademacher complexity of  $\mathcal{H}_\beta$ can be upper bounded as 
\begin{equation*}
\hat{\mathcal{R}}_n( \mathcal{H}_\beta) \leq \beta \sqrt{\frac{\kappa p}{n}}
\end{equation*}
for kernels $k$ that satisfy $k(\x,\x) \leq \kappa, \forall \x\in \mathcal{X}$.
\end{theorem}

The proof for the theorem can be found in Appendix~\ref{app:Rademacher}. 
We now make use of this result to bound the generalization error of EKL for the case where the second stage of the algorithm is kernel ridge regression (Equation~\ref{eq:EKL_opt1}). Similar results can be given  using other algorithms such as SVMs in the second stage.
\begin{corollary}\label{th:EKL_Generalization} (Generalization bound for EKL) \\[0.1cm]
	Let $M>0$. Assume that $\|f(\x)-\mathbf{y}\|_{\mathcal{Y}} \leq M$ for all $(\x,\mathbf{y})\in \mathcal{X} \times \mathcal{Y}$ and $f\in\mathcal{H}_\beta$. Then, the following holds with probability larger than $1-\delta$ over samples of size $n$ for all $f\in\mathcal{H}_\beta$:
	\begin{equation}
	\label{eq:EKL_Generalization}
	R(f) \leq \hat{R}(f) + 4 \sqrt{2} M \sqrt{\frac{\beta^2 \kappa p}{n}} + 3M  \sqrt{\frac{\log\frac{2}{\delta}}{2n}},
	\end{equation}
	where $R(f):= \mathbb{E}_{(\x,\mathbf{y})}\big[\|f(\x)-\mathbf{y}\|_{\mathcal{Y}}^2\big]$ is the expected risk w.r.t. the square loss and $\hat{R}(f) := \frac{1}{n}\sum_{i=1}^n \|f(\x_i)-\mathbf{y}_i\|_{\mathcal{Y}}^2$ is the empirical risk of $f$.
\end{corollary}	
The proof of  Corollary~\ref{th:EKL_Generalization}, which is deferred to Appendix~\ref{app:Generalization}, is based on a general Rademacher complexity learning bound provided in~\citet{bartlett2002rademacher} (see also~\citealt[Theorem~3.3]{mohri2018foundations}) and a vector-contraction inequality for the Rademacher complexity of classes of vector-valued functions proved in~\citet{maurer2016vector}.

\section{Experiments}
\label{sec:xp}

In this section we investigate the performance of our algorithm with real and artificial data sets. We start with an illustration of its behaviour, and move on to experiments on the predictive performance.
In this latter setting, we compare our proposed Entangled Kernel Learning~(EKL)\footnote{Code is available at RH's personal website, \url{riikkahuusari.com}.} algorithm to OKL~\citep{dinuzzo2011learning}; a kernel learning method for separable kernels (we use the code provided by the authors~\footnote{\url{http://people.tuebingen.mpg.de/fdinuzzo/okl.html}}), and KRR; kernel ridge regression. Furhtermore, we investigate performance of predicting with scalar-valued kernel extracted from the operator-valued kernel matrix EKL learns, and call this ptrEKL. In all the experiments we cross-validate over various regularization parameters $\lambda$, and for EKL also the $\gamma$s controlling the combination of alignments. In the experiments we consider (normalized) mean squared error (nMSE) and normalized improvement to KRR (nI) \citep{Ciliberto2015convex} as error measures. 
Finally we conclude the experimental section by comparing the performance of learning and predicting with various operator-valued kernels, showing the advantage of entangled kernels.

\subsection{Illustration}

We consider the digits data set available at scikit-learn~\footnote{\url{https://scikit-learn.org/stable/modules/generated/sklearn.datasets.load_digits.html}}~\citep{scikit-learn} and make it into multi-task data set by encoding class-memberships into vectors consisting of values ${-1, +1}$. We consider only the first four classes (digits 0-3) for clarity of illustration; in this case for example label vector $\begin{bmatrix} -1 & -1 & +1 &-1 \end{bmatrix}$ stands for digit 2. We took 25 data samples from each class as training samples, and further 25 from each for test data samples.

\begin{figure}

\centering

\includegraphics[width=0.8\linewidth]{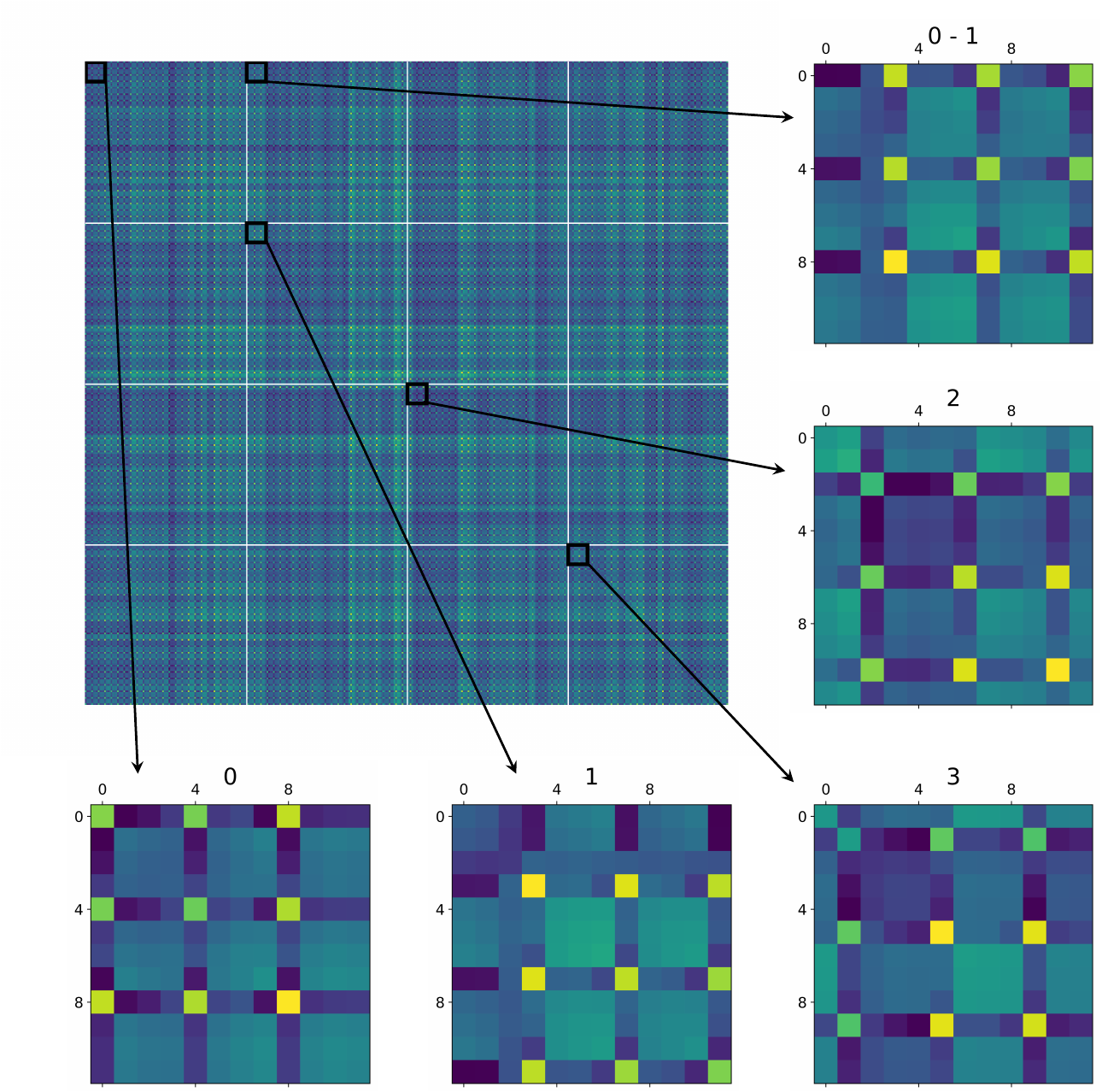}

 \caption{Small selections (bottom and right images) of the $np\times np$ entangled kernel matrix $\mathbf{G}$ (top left) trained on digits data set. The data in big matrix is ordered so, that all samples of class 0 come first, then 1, 2 and finally of 3, and for clarity of illustration white lines separate the classes. Each small selection contains 3 samples of the class(es) indicated in the title of the subfigure. For this data the output of the kernel function is a $4\times 4$ matrix. The structure of the output matrix is extremely dependent on the inputs. The off-diagonal $p\times p$ blocks of the kernel matrix are not symmetrical unlike with separable kernels.}
    \label{fig:digits_kernels}
\end{figure}

We trained EKL by restricting the number of columns in matrix $\mathbf{Q}$, $r$, to two. This drastic reduction of learnable parameters is not expected to produce the best accuracy~(as shown in the later experiments), but instead to allow us to visualize and illustrate a part on how EKL works.

In this setting the learned EKL matrix is not separable, according to the PPT condition~(end of Section~\ref{section:quantum}). Figure~\ref{fig:digits_kernels} shows parts of the entangled kernel matrix, so that in each part the samples making it up belong to one class of digits. We can see that for each class the task relationships are modelled differently. This is vastly different from the OKL approach, where the $p\times p$ output matrix has the same structure for all the data samples, as we illustrated in Figure~\ref{fig:separableIllustration1}. In our case the output matrix structure is extremely dependent on the inputs.

\begin{figure}[tb]
\centering
    \includegraphics[width=0.68\linewidth]{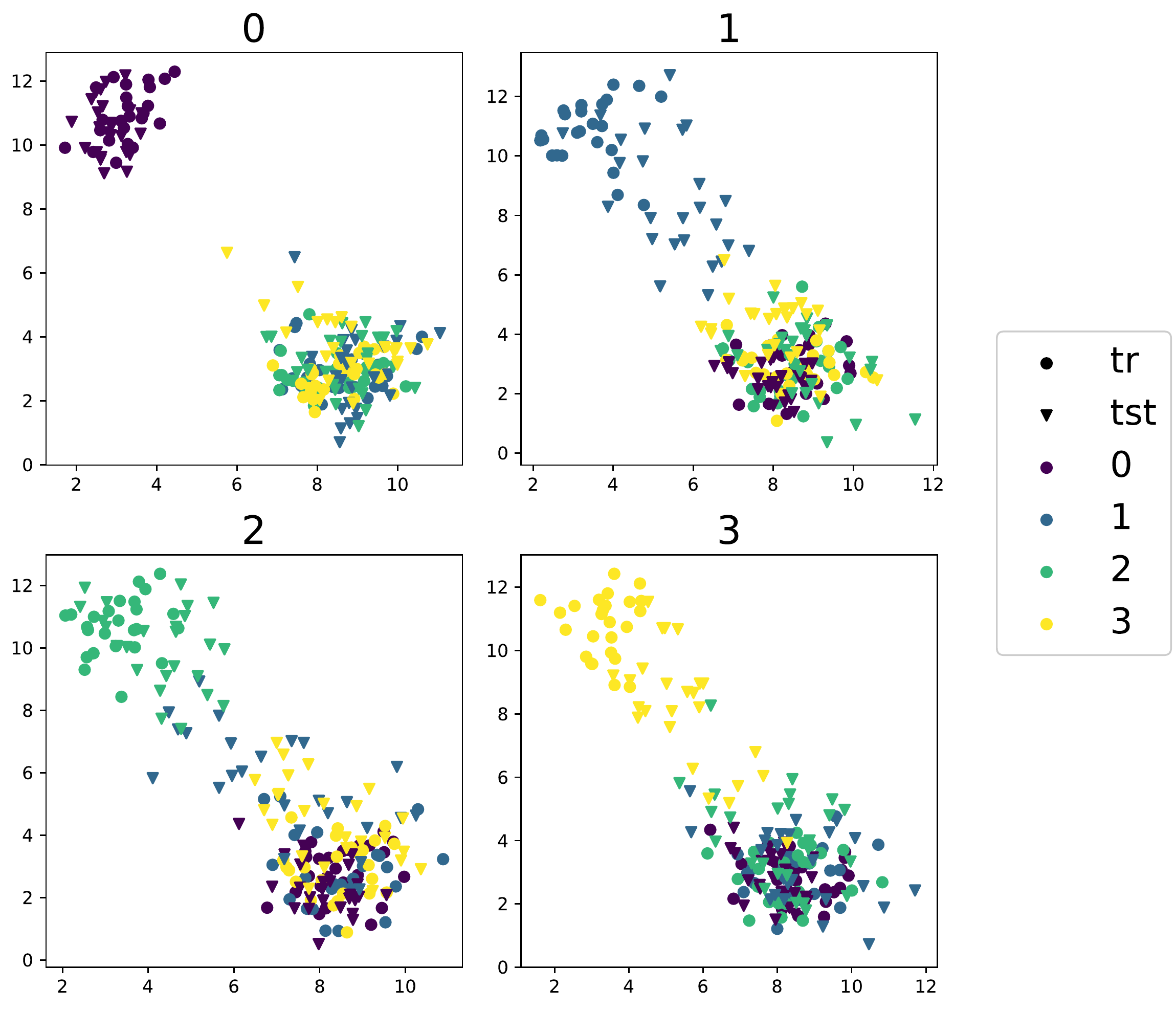}
    \caption{Supervised dimensionality reduction results on two dimensions for each of the four tasks in digits data sets, task in question indicated by the subfigure title. The classes of the projected data samples are indicated with colours, and the shape of the blob indicates whether the sample was used in training stage (round) or projected as test sample (triangle).}
    \label{fig:digits_clusters}
\end{figure}

Another way to see this dependency is via an application to supervised multi-task dimensionality reduction. We note that $\mathbf{Z} = (\hat{\bm{\Phi}}^\top \otimes \mathbf{I}_p )\mathbf{Q}$, is a matrix of size $np\times r$ and acts as an approximated feature map of the operator-valued kernel matrix, in the sense that the full kernel is $\mathbf{ZZ}^\top$. We can interpret the $\mathbf{Z}$ to give dimensionality reduction of the data with respect to all the outputs individually, as in we have $p$ dimensionality reductions of size $n\times r$. 
As a way to project data to lower dimensions, our illustrative approach closely resembles supervised dimensionality reduction~\citep{fukumizu2004dimensionality,sugiyama2006local}. 
We note that the focus of our work is not in this specific task, nor is our framework directly applicable to those works. We consider multi-output learning, making the application of EKL to this problem the first in supervised multi-task dimensionality reduction, as far as we know.

Figure~\ref{fig:digits_clusters} shows the results of the low-dimensional projection for each of the tasks. The supervised dimensionality reduction is successful for all the tasks: we can see that each task is modelled individually, as the samples corresponding to task under question are projected to a separate cluster from the other samples.
An advantage of our approach is that our model allows calculating new reduced features for new data samples; given $\hat{\bm{\Phi}}_t$ a matrix of features of new samples, the \enquote{predicted reduced features} are now simply $\mathbf{Z}_t = (\hat{\bm{\Phi}}_t^\top \otimes \mathbf{I}_p )\mathbf{Q}$. These projections are for the most part very accurate as shown in Figure~\ref{fig:digits_clusters}, respecting the task clusters especially for the easiest classes.

\subsection{Performance of EKL}

We now turn to consider the predictive performance of EKL. We first show experiments on artificial data, before turning to real-world data sets.

\subsubsection{Artificial data}

\begin{figure}[tb]
\centering
    \includegraphics[width=0.48\linewidth]{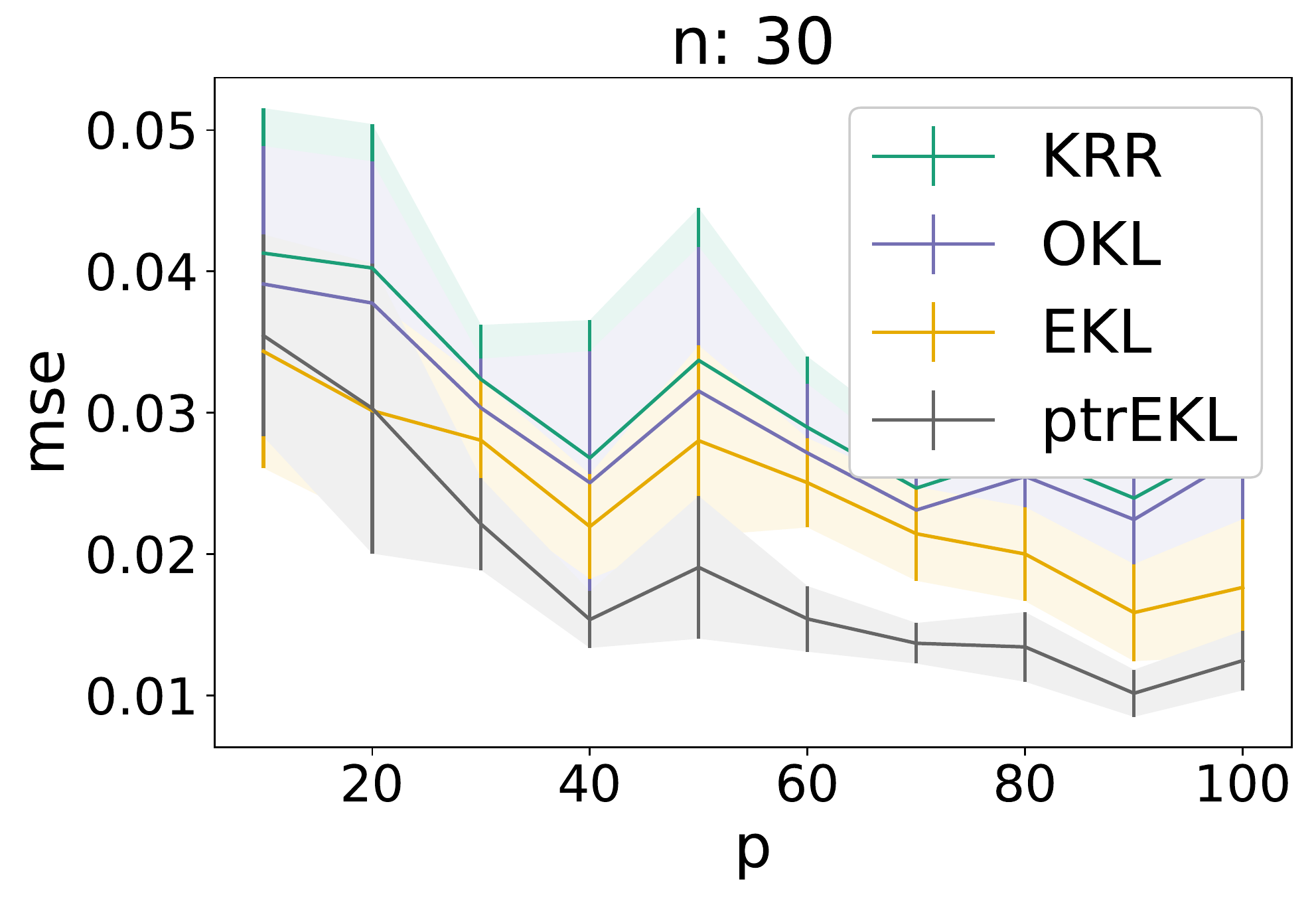}
    \includegraphics[width=0.48\linewidth]{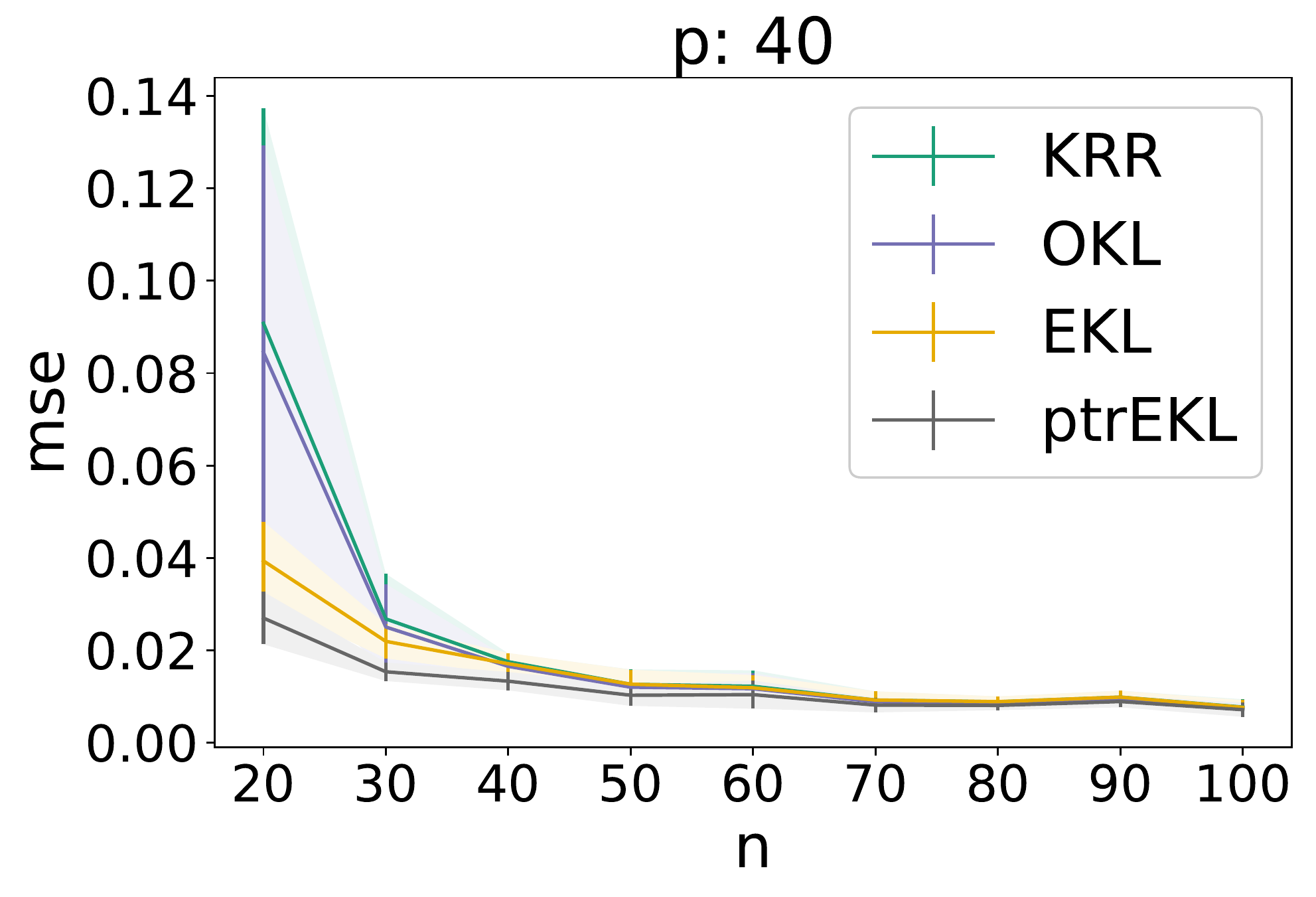}
    \caption{Results (mean squared error) of the simulated experiments with fixed amount of inputs and varying number of outputs (top), and fixed amount of outputs and varying inputs (bottom).The advantage of learning complex relationships is the biggest with small $n$.}
    \label{fig:ax_xb_c}
\end{figure}

\begin{figure}[tb]
\centering
    \includegraphics[width=0.48\linewidth]{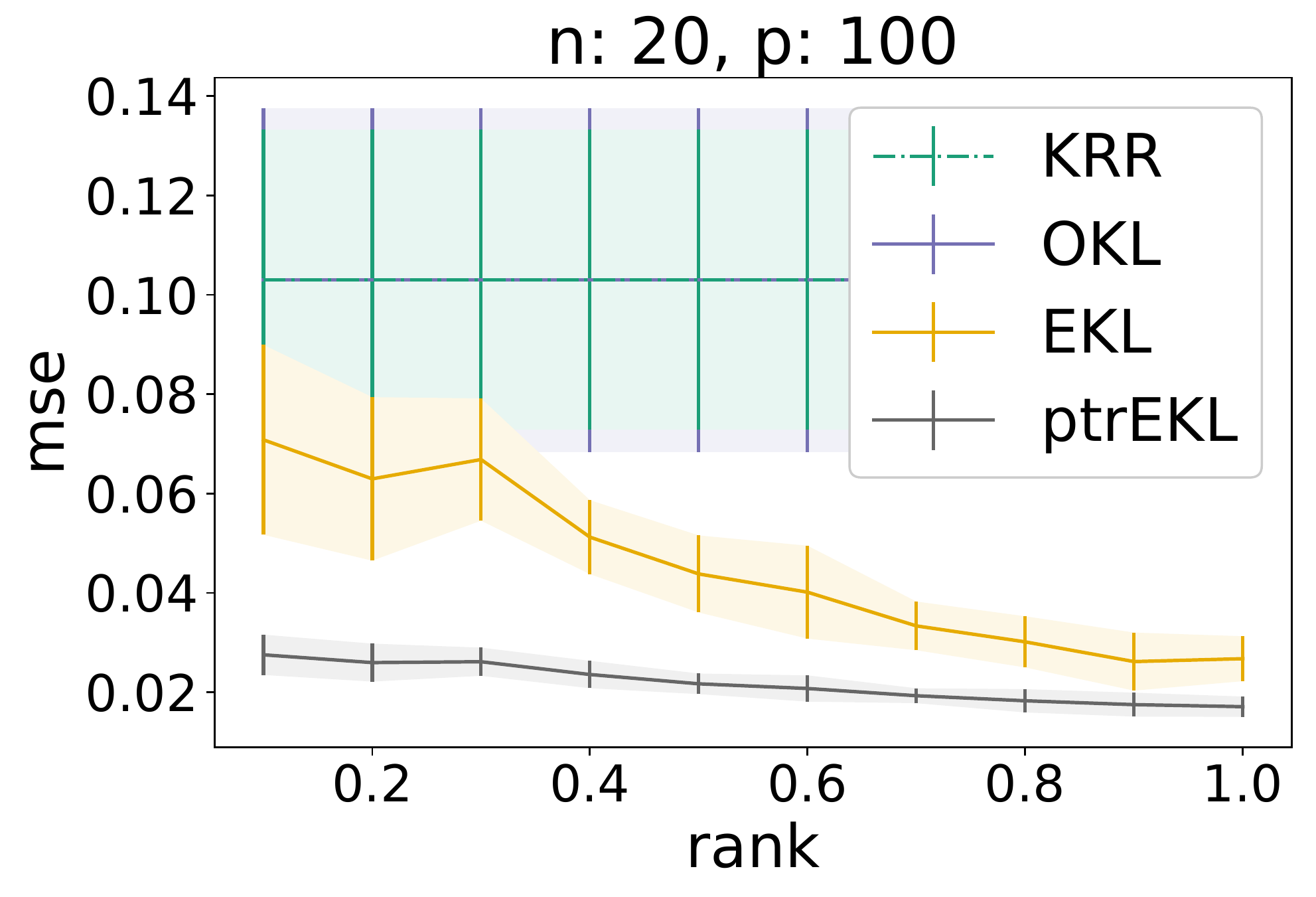}
    \caption{Results (mean squared error) on the simulated data set with varying number of columns used in learning $\mathbf{Q}$ with $n=20$ and $p=100$. The number of columns, or rank level, is shown as percentage of the full possible rank. OKL And KRR are plotted as a reference, naturally the rank constraint of EKL does not affect these methods.}
    \label{fig:ax_xb_c_rank}
\end{figure}

EKL is expected to learn complex relationships within the data. To illustrate this, we created data with bi-linear model $\mathbf{TCA+ICK = Y}$, where $\mathbf{T}$, $\mathbf{C}$ and $\mathbf{A}$ are randomly created $p\times p$, $p\times n$, and $n\times n$ matrices respectively. $\mathbf{K}$ is linear kernel calculated from randomly generated data $\mathbf{X} \in \R^{n\times d}$; this scalar-valued kernel is given to all the learning algorithms along with noisy labels $\mathbf{Y}$. 

The results are shown in Figure~\ref{fig:ax_xb_c}. We can see that when $p$ is larger than $n$ (or comparable) the predictive capabilities of EKL are much better than for other methods. Here predicting with the scalar-valued kernel extracted form learned entangled kernel gives the best results.

We also investigated the effect the choice of rank $r$ (number of columns in matrix $\mathbf{Q}$) has for EKL performance~(Figure~\ref{fig:ax_xb_c_rank}). As the rank increases, the performance of EKL gets better. This is true to an extent also for ptrEKL, however there seems not to be as strong effect as with EKL. This is not so surprising; ptrEKL has fewer parameters affecting predictive performance, so decreasing the amount of them shouldn't change the results as much as for full EKL.

\subsubsection{Real data}

\begin{table*}[tb]
\centering
\resizebox{\textwidth}{!}{
\begin{tabular}{lrrrrrrrr}  
\toprule
 & \multicolumn{2}{r}{$n=50$ \hspace*{0.55cm} }  & \multicolumn{2}{r}{$n=100$  \hspace*{0.4cm} }  & \multicolumn{2}{r}{$n=200$  \hspace*{0.4cm} }  & \multicolumn{2}{r}{$n=1000$  \hspace*{0.3cm} } \\
method  &  nMSE & \cellcolor[gray]{0.8} nI &  nMSE & \cellcolor[gray]{0.8}nI &  nMSE & \cellcolor[gray]{0.8}nI &  nMSE & \cellcolor[gray]{0.8}nI \\
\midrule
KRR & 0.2418 $\pm$ 0.0281   &  \cellcolor[gray]{0.8} 0.0000  & 0.1668 $\pm$ 0.0097   &  \cellcolor[gray]{0.8} 0.0000  & 0.1441 $\pm$ 0.0037   &  \cellcolor[gray]{0.8} 0.0000  & 0.1273 $\pm$ 0.0006   &  \cellcolor[gray]{0.8} 0.0000  \\
OKL & 0.2445 $\pm$ 0.0296   &  \cellcolor[gray]{0.8} -0.0109  & 0.1672 $\pm$ 0.0099   &  \cellcolor[gray]{0.8} -0.0026  & 0.1442 $\pm$ 0.0037   &  \cellcolor[gray]{0.8} -0.0009  & 0.1273 $\pm$ 0.0006   &  \cellcolor[gray]{0.8} -0.0000  \\
EKL/ptrEKL & 0.2381 $\pm$ 0.0250   &  \cellcolor[gray]{0.8} 0.0139  & 0.1661 $\pm$ 0.0097   &  \cellcolor[gray]{0.8} 0.0040  & 0.1440 $\pm$ 0.0037   &  \cellcolor[gray]{0.8} 0.0003  & 0.1273 $\pm$ 0.0006   &  \cellcolor[gray]{0.8} 0.0001  \\
\bottomrule
\end{tabular}}
\caption{Results on Sarcos data set with various number of training samples used, averaged over 10 data partitions. The advantage of learning complex relationships decreases with amount of data samples increasing.} %
\label{tab:sarcos}  
\end{table*}

\begin{table}[tb]
\centering
\resizebox{\columnwidth}{!}{
\begin{tabular}{lrrrrrr}  
\toprule
 & \multicolumn{2}{r}{$n=5$ \hspace*{0.55cm} }  & \multicolumn{2}{r}{$n=10$  \hspace*{0.4cm} }  & \multicolumn{2}{r}{$n=15$  \hspace*{0.4cm} } \\
 method & nMSE & \cellcolor[gray]{0.8} nI  & nMSE & \cellcolor[gray]{0.8} nI  & nMSE & \cellcolor[gray]{0.8} nI  \\
\midrule
KRR & 0.951 $\pm$ 0.101   &  \cellcolor[gray]{0.8} 0.000  & 0.813 $\pm$ 0.141   &  \cellcolor[gray]{0.8} 0.000 &  0.761 $\pm$ 0.037   & \cellcolor[gray]{0.8}  0.000 \\
OKL & 1.062 $\pm$ 0.250   &  \cellcolor[gray]{0.8} -0.092  & 0.900 $\pm$ 0.196   &  \cellcolor[gray]{0.8} -0.094 & 0.788 $\pm$ 0.058   & \cellcolor[gray]{0.8}  -0.034 \\
EKL/ptrEKL & 0.840 $\pm$ 0.084   &  \cellcolor[gray]{0.8} 0.124  & 0.722 $\pm$ 0.036   &  \cellcolor[gray]{0.8} 0.107 &  0.728 $\pm$ 0.033   & \cellcolor[gray]{0.8}  0.044 \\
\bottomrule
\end{tabular} }
\caption{Results on Weather data set averaged over 5 data partitions. }
\label{tab:weather}
\end{table}

\begin{figure}[p]
\centering
    \includegraphics[width=0.48\linewidth]{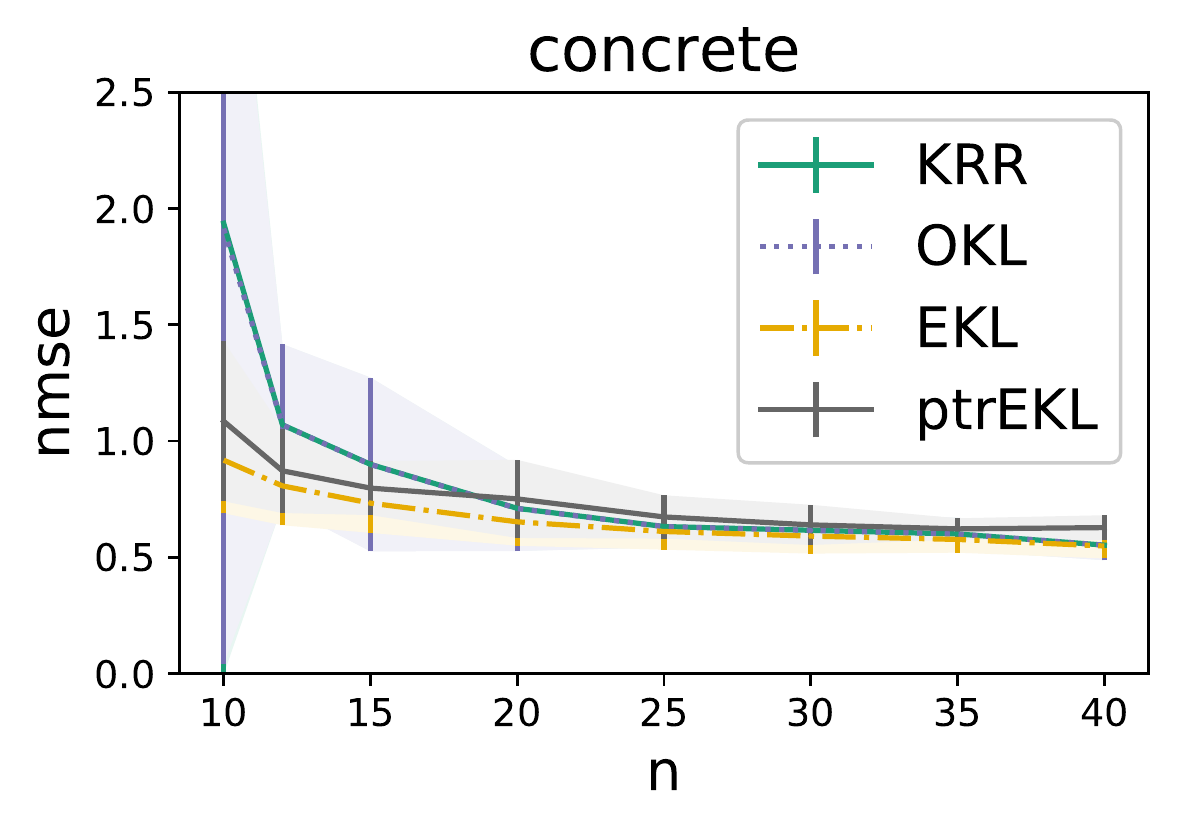}
    \caption{Results (normalized mean squared error) on the concrete data set with varying amount of training data ($n$) used. The advantage of learning complex relationships is biggest on small $n$.}
    \label{fig:concrete}
\end{figure}

\begin{figure}[p]
\centering
    \includegraphics[width=0.48\linewidth]{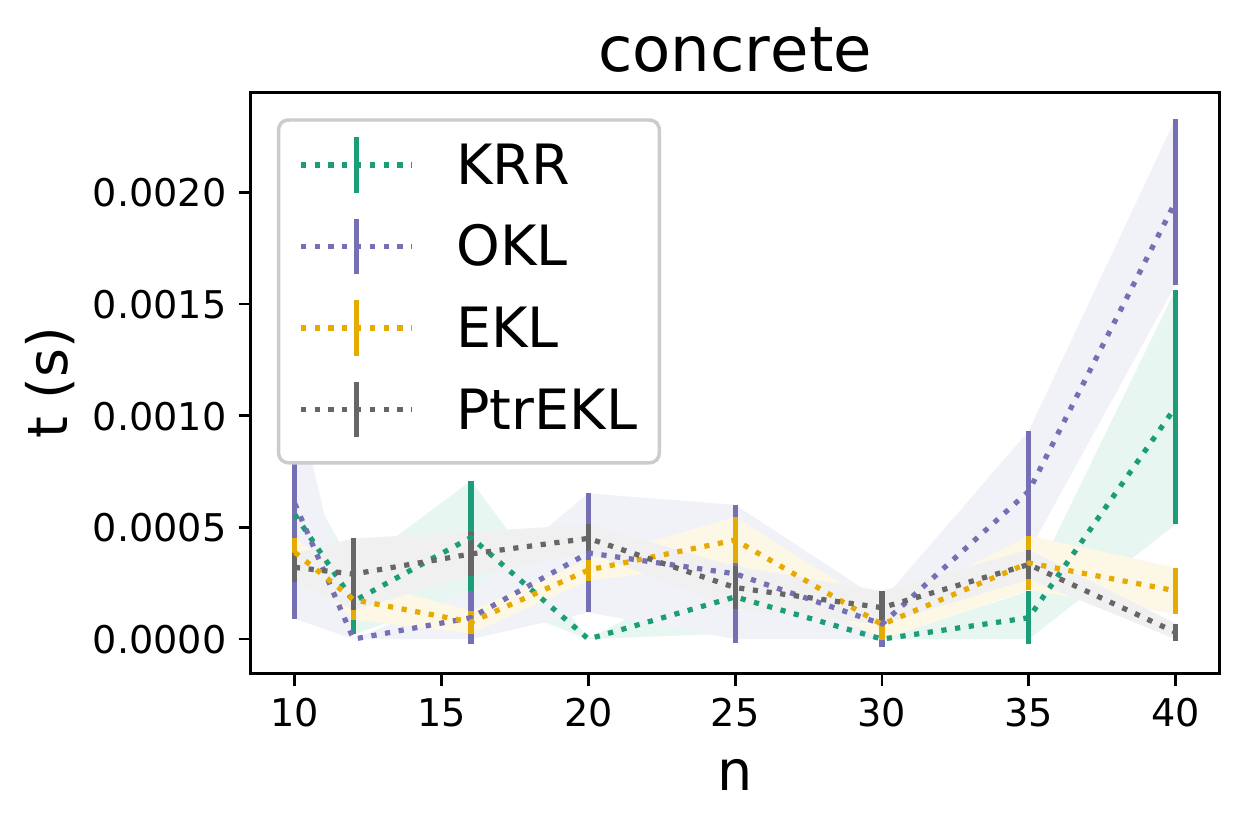}
    \caption{Running times in seconds (including calculating the predictive function given a kernel and predicting) on the concrete data set with varying amount of training data ($n$) used.} 
    \label{fig:concrete_times}
\end{figure}

\begin{table}[p]
\centering
\begin{tabular}{lrrrrrr}  
\toprule
 & \multicolumn{2}{r}{$n=12$ \hspace*{0.55cm} }  & \multicolumn{2}{r}{$n=20$  \hspace*{0.4cm} }  & \multicolumn{2}{r}{$n=40$  \hspace*{0.4cm} } \\
 method & nMSE & \cellcolor[gray]{0.8} nI  & nMSE & \cellcolor[gray]{0.8} nI  & nMSE & \cellcolor[gray]{0.8} nI  \\
\midrule
KRR &  1.070 $\pm$ 0.347   &  \cellcolor[gray]{0.8} 0.000  & 0.710 $\pm$ 0.183   &  \cellcolor[gray]{0.8} 0.000  & 0.552 $\pm$ 0.065   &  \cellcolor[gray]{0.8} 0.000  \\
OKL & 1.069 $\pm$ 0.347   &  \cellcolor[gray]{0.8} 0.001  & 0.710 $\pm$ 0.183   &  \cellcolor[gray]{0.8} 0.001  & 0.552 $\pm$ 0.064   &  \cellcolor[gray]{0.8} 0.000  \\
EKL & 0.796 $\pm$ 0.164   &  \cellcolor[gray]{0.8} 0.266  &  0.634 $\pm$ 0.103   &  \cellcolor[gray]{0.8} 0.097 & 0.547 $\pm$ 0.046   &  \cellcolor[gray]{0.8} 0.007  \\
ptrEKL & 0.843 $\pm$ 0.186   &  \cellcolor[gray]{0.8} 0.212 & 0.726 $\pm$ 0.170   &  \cellcolor[gray]{0.8} -0.023  & 0.627 $\pm$ 0.058   &  \cellcolor[gray]{0.8} -0.130 \\
\bottomrule
\end{tabular} 
\caption{Results on Concrete data set, averaged over 10 data partitions. } 
\label{tab:concrete}
\end{table}

We consider the following regression data sets: Concrete slump test\footnote{UCI data set repository} with 103 data samples and three output variables; Sarcos~$\!$\footnote{\texttt{www.gaussianprocess.org/gpml/data/}} is a data set characterizing robot arm movements with 7 tasks; Weather~$\!\!$\footnote{\texttt{https://www.psych.mcgill.ca/misc/fda/}} has daily weather data ($p=365$) from 35 stations. With these data sets, we considered linear kernels and used the original features in EKL, and full rank in learning $\mathbf{Q}$.   %
Furthermore, we consider the uWaveGesture data set\footnote{\texttt{http://www.cs.ucr.edu/$\sim$eamonn/time\_series\_data/}}, a multi-view data set for classification, which we use in a setting of predicting a view from another, giving us a regression problem with 314 output variables. We again consider linear kernels with the data set, and investigate also the effect of chosen rank of a matrix $\mathbf{Q}$.

\begin{table}[tb]
\centering
\begin{tabular}{llccc}  
\toprule
 &  & $n=20$  & $n=30$  & $n=40$ \\
\midrule
\multirow{4}{1.8cm}{\centering Class 1} & KRR &  3.012 $\pm$ 0.2337 & 2.853 $\pm$ 0.0780 &  2.262 $\pm$ 0.3157 \\
 & OKL &  3.403 $\pm$ 0.0899 & 3.232 $\pm$ 0.1877 &  2.440 $\pm$ 0.3451 \\
 & EKL &  1.136 $\pm$ 0.0937 & 1.123 $\pm$ 0.0852 &  1.107 $\pm$ 0.1019 \\
 & ptrEKL &  1.323 $\pm$ 0.1807 & 1.208 $\pm$ 0.0932 &  1.132 $\pm$ 0.1221 \\
\midrule
\multirow{4}{1.8cm}{\centering Class 2} & KRR &  1.773 $\pm$ 0.0655 &  2.008 $\pm$ 0.3391 &  1.937 $\pm$ 0.3170 \\
 & OKL &  1.802 $\pm$ 0.0603 &  2.096 $\pm$ 0.3673 &  2.074 $\pm$ 0.3483 \\
 & EKL &  0.991 $\pm$ 0.0517 &  0.943 $\pm$ 0.0237 &  0.908 $\pm$ 0.0130 \\
 & ptrEKL &  1.026 $\pm$ 0.0509 &  0.940 $\pm$ 0.0143 &  0.914 $\pm$ 0.0138 \\
\midrule
\multirow{4}{1.8cm}{\centering Class 3} & KRR & 1.081 $\pm$ 0.1437 & 1.028 $\pm$ 0.1235 & 0.902 $\pm$ 0.1020\\
 & OKL & 1.173 $\pm$ 0.1720 & 1.146 $\pm$ 0.1753 & 0.993 $\pm$ 0.1486\\
 & EKL & 0.671 $\pm$ 0.0202 & 0.638 $\pm$ 0.0214 & 0.632 $\pm$ 0.0325\\
 & ptrEKL & 0.681 $\pm$ 0.0226 & 0.651 $\pm$ 0.0345 & 0.632 $\pm$ 0.0307\\
\bottomrule
\end{tabular} 
\caption{Normalized mean squared errors over three runs on the uWaveGesture data set, classes 1-3, with varying amount of training data and full rank of the EKL. }
\label{tab:uWave1}

\centering
\begin{tabular}{llccc}  
\toprule
 &  & rank $100\%$  & rank $60\%$    & rank $20\%$  \\
\midrule
\multirow{2}{1.8cm}{\centering Class 1} & EKL & 1.123 $\pm$ 0.0852 & 1.134 $\pm$ 0.0762 & 1.171 $\pm$ 0.1199 \\
 & ptrEKL & 1.208 $\pm$ 0.0932 & 1.228 $\pm$ 0.0970 & 1.325 $\pm$ 0.1630 \\
\midrule
\multirow{2}{1.8cm}{\centering Class 2} & EKL &  0.943 $\pm$ 0.0237 &  0.939 $\pm$ 0.0091 &  0.941 $\pm$ 0.0074 \\
 & ptrEKL &  0.940 $\pm$ 0.0143 &  0.958 $\pm$ 0.0053 &  0.987 $\pm$ 0.0137 \\
\midrule
\multirow{2}{1.8cm}{\centering Class 3} &  EKL &  0.638 $\pm$ 0.0214 &  0.644 $\pm$ 0.0222 &  0.654 $\pm$ 0.0271 \\
& ptrEKL &  0.651 $\pm$ 0.0345 &  0.647 $\pm$ 0.0344 &  0.647 $\pm$ 0.0394 \\

\bottomrule
\end{tabular} 
\caption{Normalized mean squared errors over three runs with $n=30$ on the uWaveGesture data set, classes 1-3, with varying rank of the EKL. }
\label{tab:uWave2}
\end{table}

The main advantage of learning complex dependencies in the data lies in the setting where number of samples is relatively low; a phenomenon observed already in output kernel learning setting \citep{Ciliberto2015convex,jawanpuria2015efficient}. 
With small amounts of data learning the complex relationships in EKL is even more beneficial than learning the output dependencies of OKL. 
Figure \ref{fig:concrete} shows this advantage on Concrete data set when number of instances used in training is small. Here, in contrast to our simulated data, EKL performs better than ptrEKL.
For Sarcos data set we consider the setting in~\cite{Ciliberto2015convex} and show the results in Table~\ref{tab:sarcos} (predicting is done to all 5000 test samples). 
As can be expected, the results with the Sarcos data set which has very few outputs ($p=7$) do not show improvement over the compared methods when the number of samples is large. However we can ascertain that our EKL finds the same solutions than the other methods with the large sample sizes - indeed the methods perform identically when $n$ increases. We expect that the main improvement of the  EKL lies in the cases when $n$ and $p$ are comparable, or $p>n$.
This is clearly seen with the Weather data set, where number of outputs is much larger than the number of data samples (Table \ref{tab:weather}).

While we investigate the running times of the compared algorithms more in depth in next subsection, we here in Figure~\ref{fig:concrete_times} display them for the Concrete data set. The Figure shows the combined time for both calculating the predictive function given the kernel, and predicting.

To show further the advantage of our method in the case where $n<p$, we turn our attention to uWaveGesture data set with three views. Each of the views is of dimensionality 314. The training set partition contains in total 896 data samples from the three views and eight classes. However as we want to investigate the case with large distinction between $n$ and $p$, we only consider the smaller sets of data belonging to each class, with around 100 samples each which we divide into training and testing sets for the compared algorithms. 
Tables~\ref{tab:uWave1} and~\ref{tab:uWave2} show the results. In Table~\ref{tab:uWave1} we present the errors for the first three classes with full rank parameter $r$ in the EKL algorithms. This setting is very extreme with the tiny sample size and large $p$, and it is clear that our method obtains the best result. We also detail in Table~\ref{tab:uWave2} the EKL results with respect to the rank parameter $r$ on 100\%, 60\% and 20\% of the full rank. Even though the error rises a little in most cases, it is not significant especially compared to the results with KRR and OKL. This further justifies the use of our algorithm in the more efficient, low-rank setting.


\subsection{Running Times of Learning with OvKs}

\begin{figure}[tb]

\centering

\includegraphics[width=0.45\linewidth]{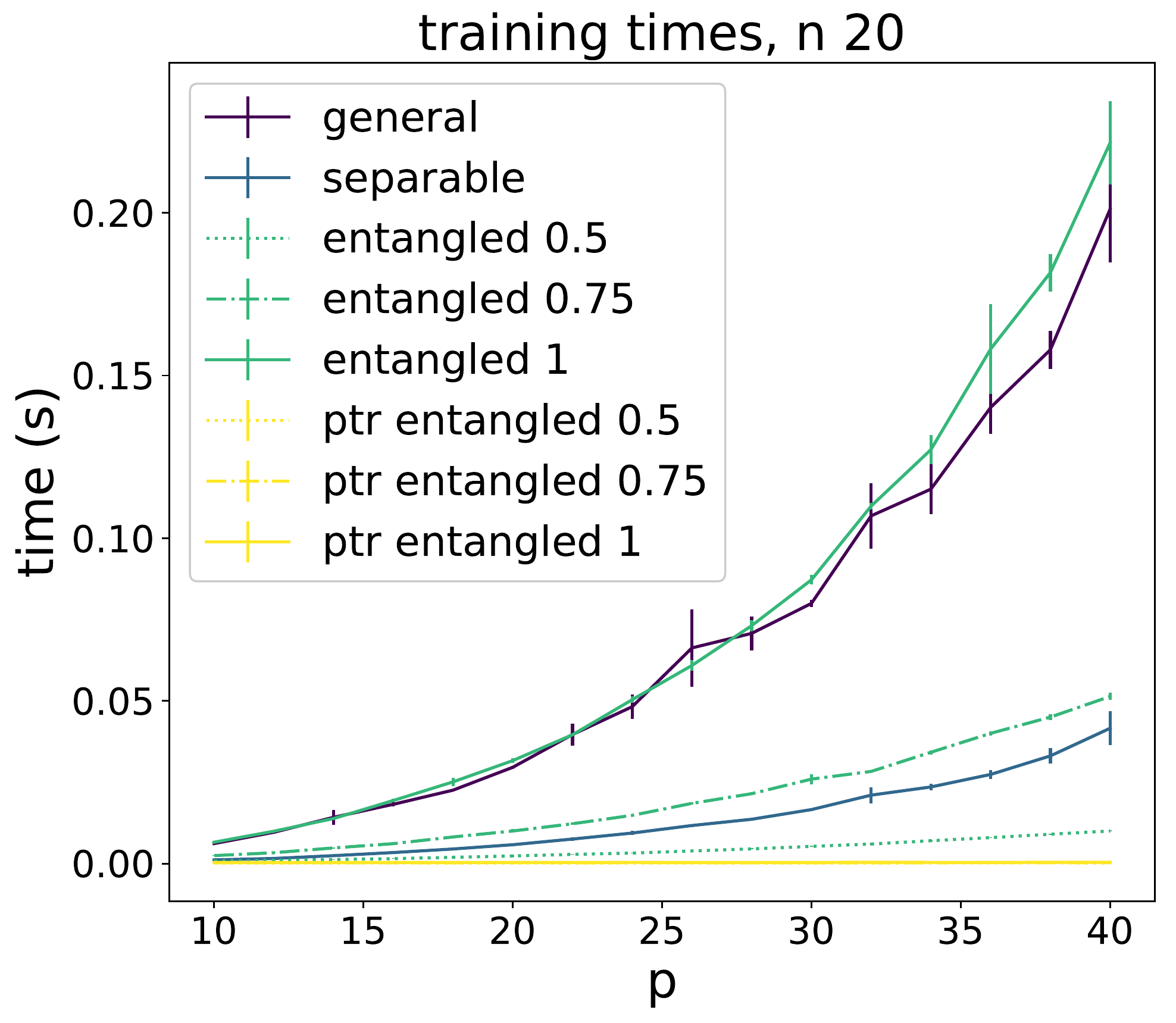}
\includegraphics[width=0.45\linewidth]{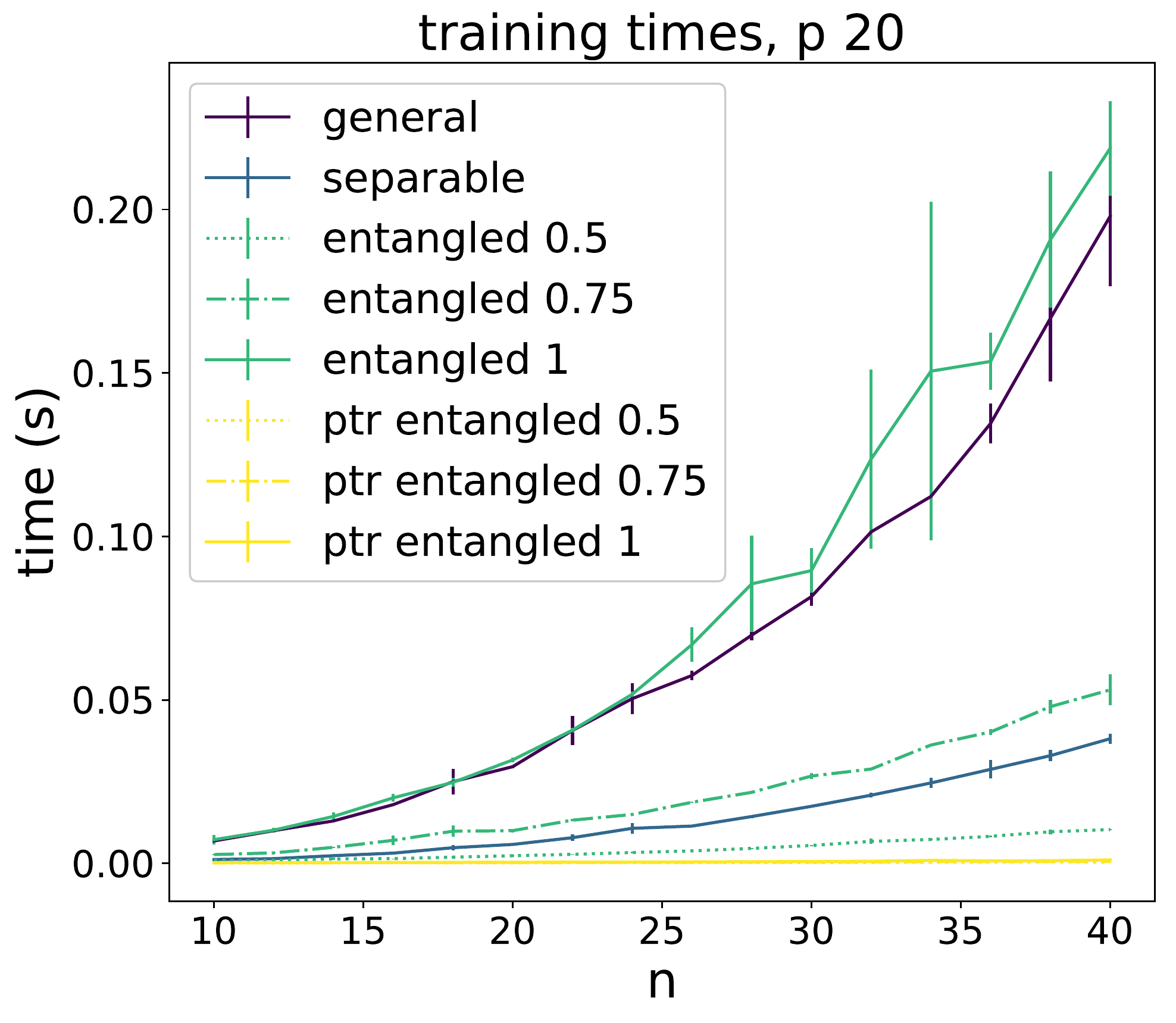}
\includegraphics[width=0.45\linewidth]{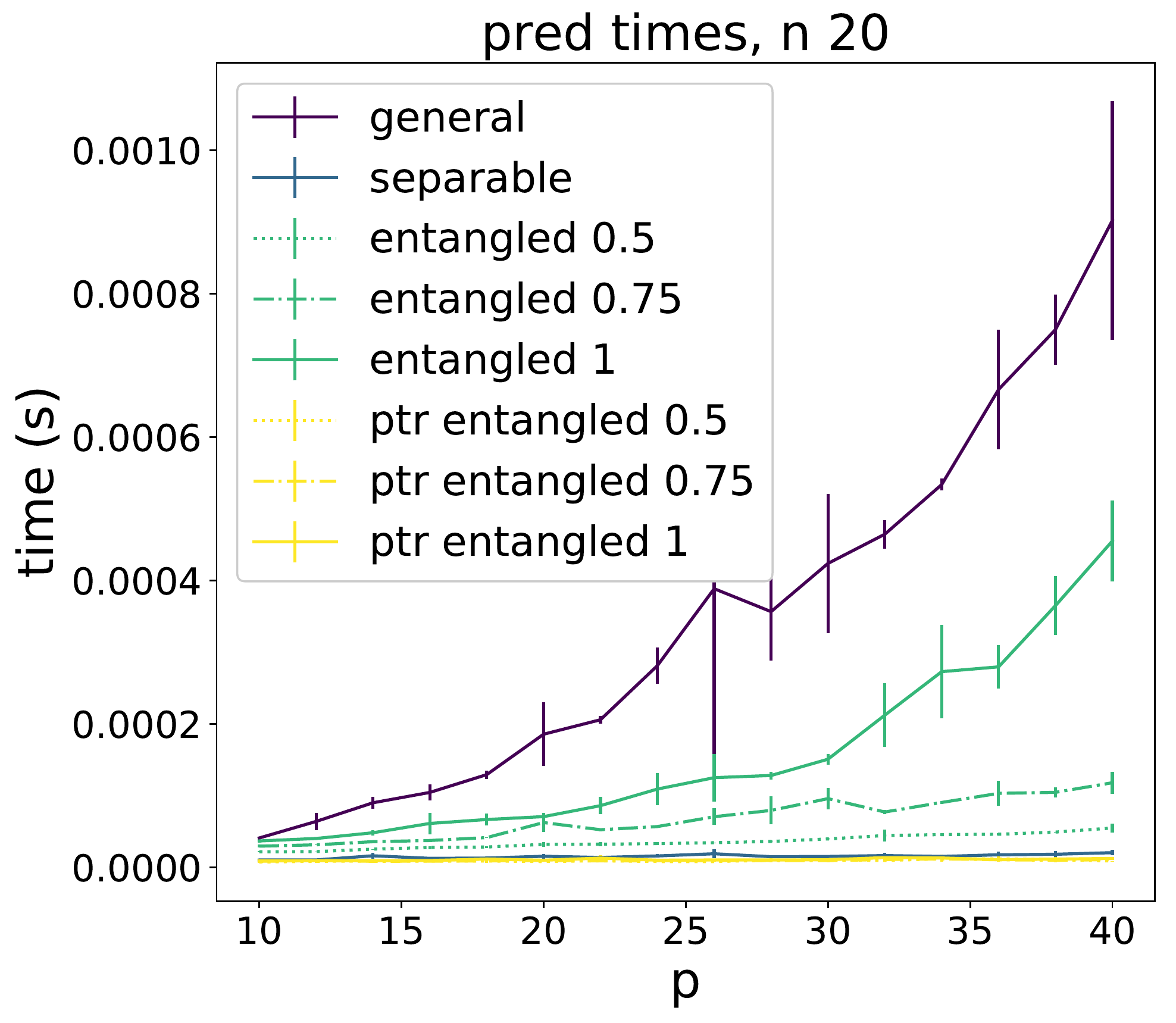}
\includegraphics[width=0.45\linewidth]{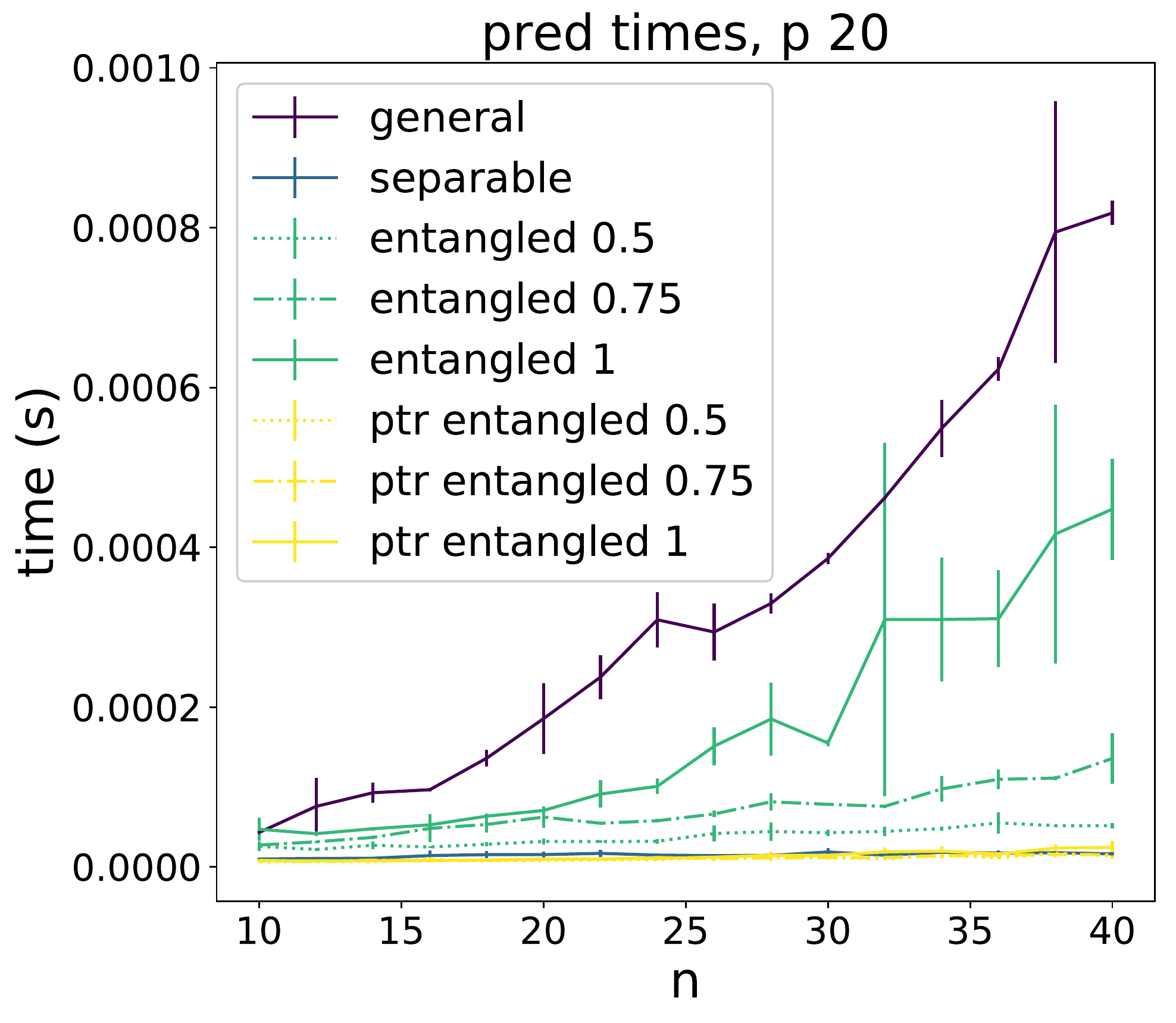}

\caption{Times for training (top row) or predicting (bottom row) with the various operator-valued kernels with fixed $n$/$p$ indicated in the subtitle of the figures. For entangled and partial trace entangled we show the results with various choices of $m$ and $r$, indicated with the number after method name in the legend. For example \enquote{entangled 0.75} means that $m=0.75\cdot n$ and $r=0.75\cdot mp$, meaning both are 75\% of the largest possible value. $t=n$ for all the experiments.}\label{fig:times_all}

\end{figure}

\begin{figure}[tb]

\centering

\includegraphics[width=0.45\linewidth]{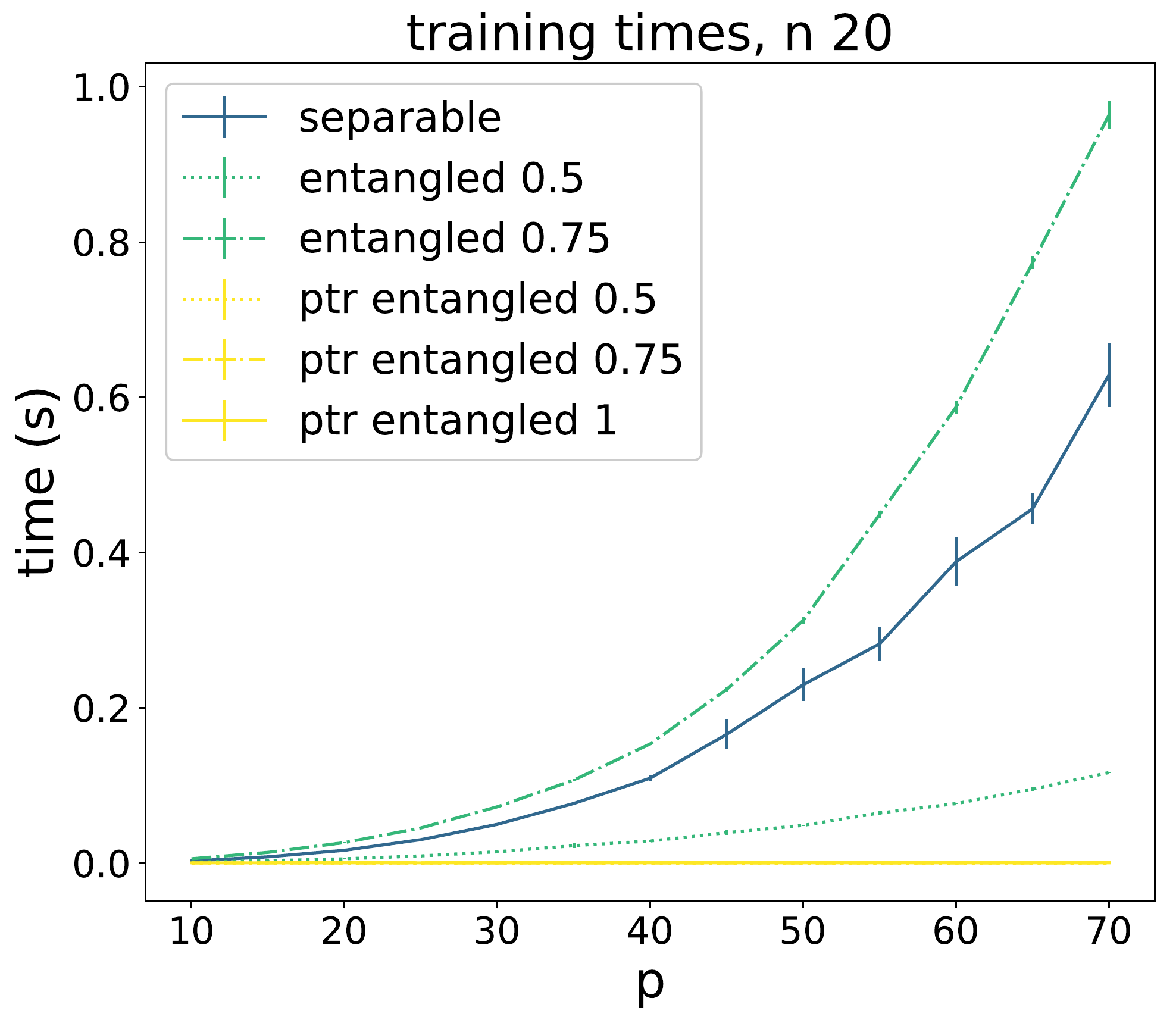}
\includegraphics[width=0.45\linewidth]{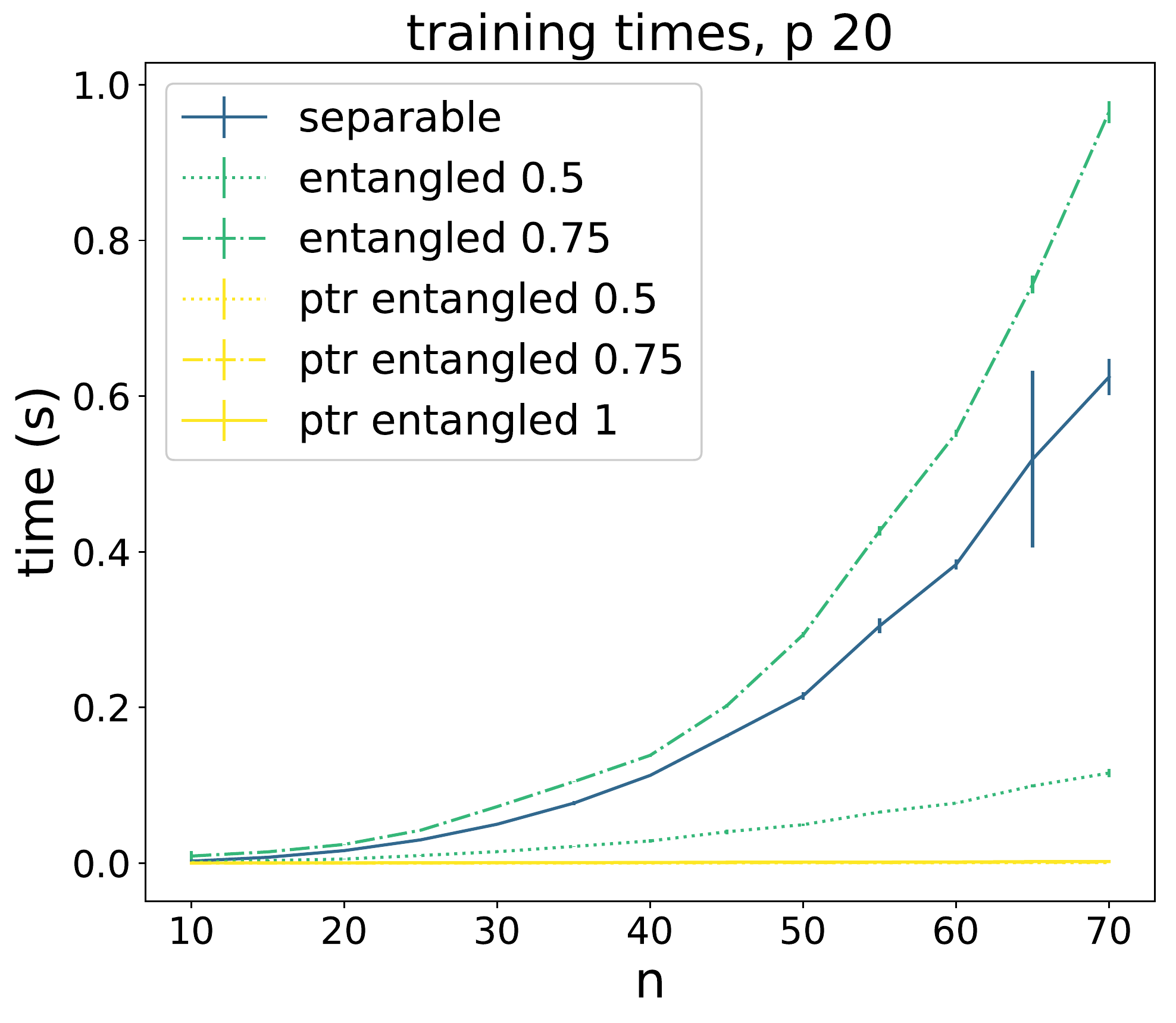}
\includegraphics[width=0.45\linewidth]{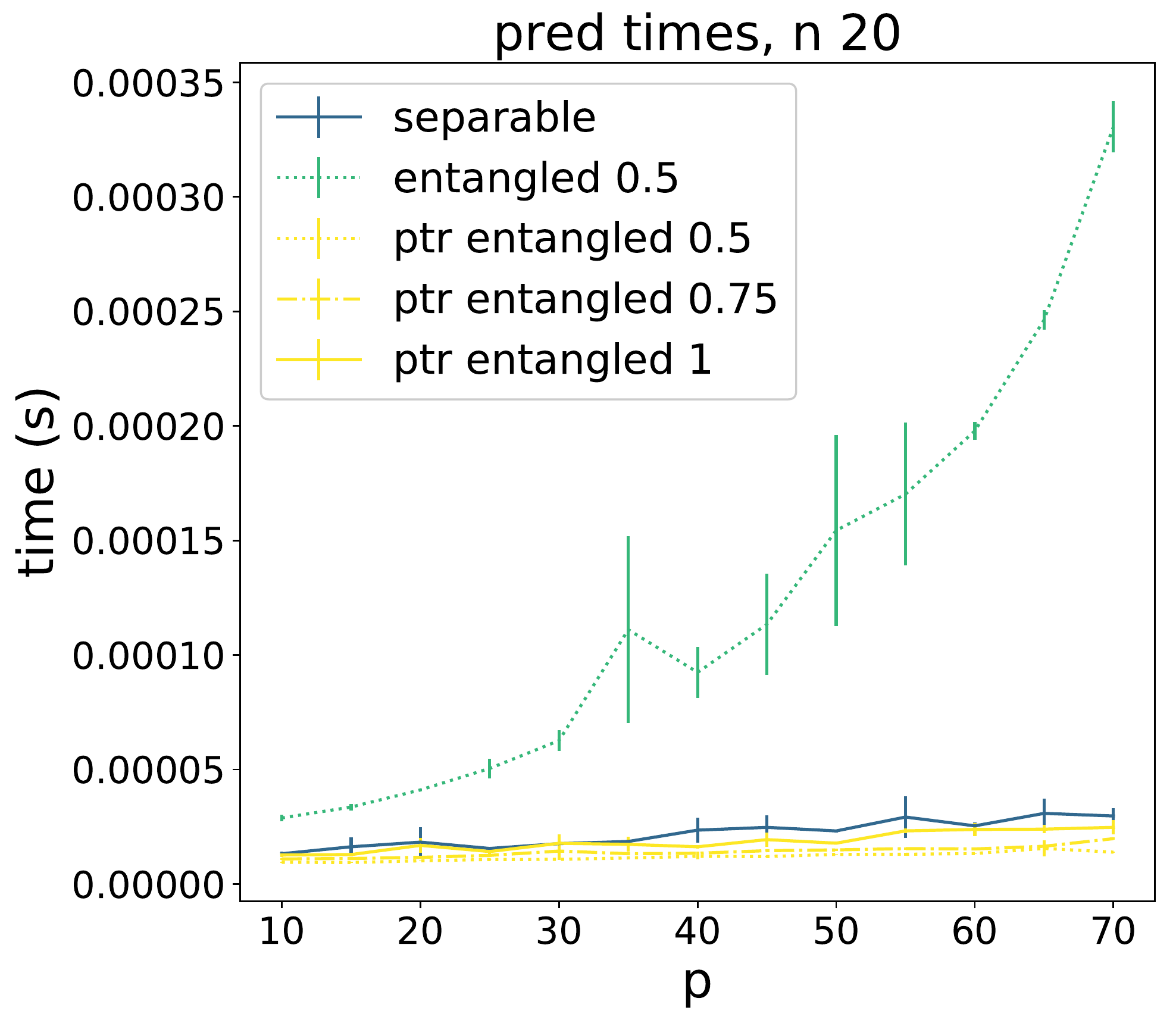}
\includegraphics[width=0.45\linewidth]{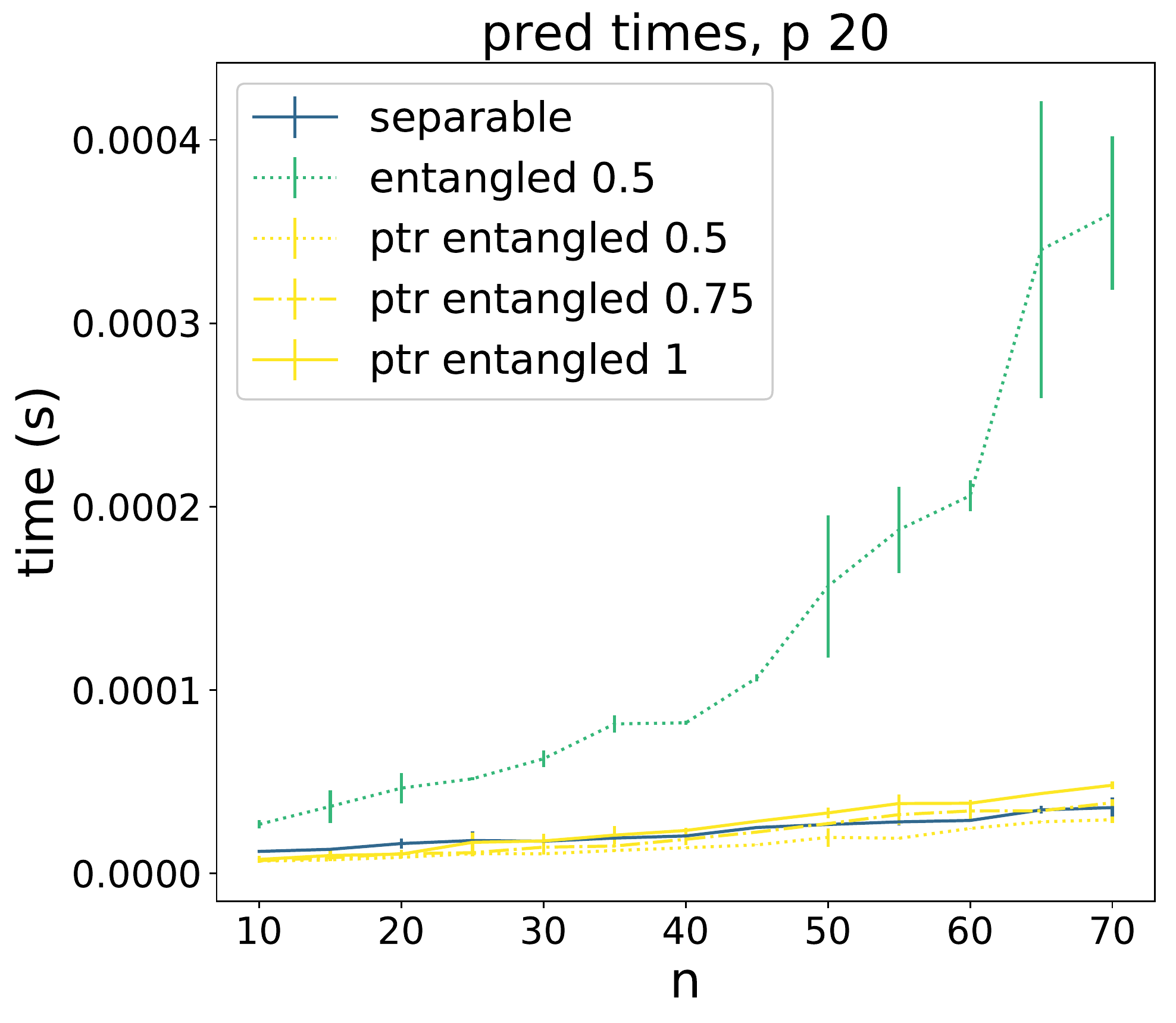}

\caption{A closer look at the times for training (top row) or predicting (bottom row) with the various operator-valued kernels with fixed $n$/$p$ indicated in the subtitle of the figures. For entangled and partial trace entangled we show the results with various choices of $m$ and $r$, indicated with the number after method name in the legend. For example \enquote{entangled 0.75} means that $m=0.75\cdot n$ and $r=0.75\cdot mp$, meaning both are 75\% of the largest possible value. $t=n$ for all the experiments.}\label{fig:times_closer}

\end{figure}

As we show in Table~\ref{tb:complexities}, there is a big difference in computational complexities between various operator-valued kernels. Namely, the complexities under question are those of calculating the parameters $\mathbf{c}$ of the predictive function, and of calculating the label predictions.
In this section we perform experiments to highlight these differences. 
We note that here we do not learn the kernels, but only compare in the experiments the differences on pre-defined entangled/separable or general operator-valued kernels.

In our experiments we created random data (random kernel matrices, random $\mathbf{c}$ from which labels were calculated) with which we performed our calculations. 
We repeated our experiments five times with different random data, and each time timed the execution of learning and predicting five times. In our results (Figures~\ref{fig:times_all} and~\ref{fig:times_closer}) we present averages of these runs. 

Figure~\ref{fig:times_all} shows the times of the runs for all the methods listed in Table~\ref{tb:complexities}. As expected, the operator-valued kernels with no structure are among the worst in every run, with scalar-valued kernels much better than them. 
The performance of entangled kernels naturally rests on the crucial parameters $m$ and $r$. In the experiments we modify both at the same time, from full values to 75\% and 50\% of the full value. With maximum $m$ and $r$ entangled kernels are as slow as the general operator-valued kernels, but outperform them in predictive step. With 75\% they train comparably to separable kernels, while outperforming them with 50\%. Partial trace version of entangled kernels naturally has the lowest computational cost in training step. In predicting, entangled kernels are always better than general operator-valued kernels. Separable and partial trace kernels perform similarly, which can be seen better in Figure~\ref{fig:times_closer} showing a closer look to separable and entangled kernels.

\section{Conclusion}
\label{sec:conclusion}

In this work we shed new light on meaning of inseparable kernels by defining a general framework for constructing operator-valued kernels based on the notion of partial trace and using ideas borrowed from the field of quantum computing.
Instances of our framework include entangled kernels, a new conceptually interesting class of kernels that is designed to capture more complex dependencies between input and output variables as the more restricted class of separable kernels. 
We have proposed a new algorithm, entangled kernel learning~(EKL), that learns this entangled kernel in kernel alignment framework. The first step uses a definition of kernel alignment, extended here for use with operator-valued kernels with help of partial trace operator. 
In contrast to output kernel learning~(OKL), EKL is able to learn inseparable kernels and can model a larger variety of interactions between input and output data. 
Moreover, the structure of the entangled kernels enables more efficient computation than that with general operator-valued kernels. 
Our illustration on artificial data and experiments on real data give validation to our approach. 

Like with previous work with separable kernels, the main advantage of learning complex relations in the data lies in setting where number of data samples is relatively low. 
In the experimental section, we observed 
that difference in performance between EKL and OKL decreases as sample size increases. We think that one possible reason for that is that number of outputs relative to samples is not as large in this case, and it will be interesting to thoroughly investigate EKL in the setting where the number of outputs is even larger. Moreover, the effect of choosing the number of columns in matrix Q (or choosing the rank of the entangled kernel) would warrant further study, especially with small values (low rank kernel setting). It is also well-known that feature representation of a kernel is not unique, and we leveraged this in our work. Studying the effect on this could be worthwhile, as well as any theoretical consequences it has.

Using the Kraus representation, entangled kernels can be viewed as a mapping from a covariance matrix on the input features to a covariance matrix on the output labels. Recently, Riemannian networks for symmetric positive definite (SPD) learning have been introduced in~\cite{huang2017riemannian}. These networks receive SPD matrices, such as covariance descriptors, as inputs, and preserve their SPD structure across the layers. 
The SPD structure is preserved via the bilinear mapping (BiMap) layer, which 
can be seen as a particular case of a Choi-Kraus representation with Kraus rank equal to 1. Generalizing  SPD networks to  higher Kraus ranks could provide a way to efficiently learn entangled kernels with deep learning machinery, and would be an interesting direction
to explore in future work.

Overall, we hope that our work featuring the first comprehensive description on how to learn non-separable operator-valued kernels will give a boost to the field of learning inseparable kernels.



\acks{We thank the anonymous reviewers for their helpful comments and suggestions, which improved the quality of the paper.
	%
	We also thank F. Denis, P. Arrighi and G. Di Molfetta for fruitful discussions. 
This work has been funded by the French National Research Agency~(ANR) project QuantML (grant number ANR-19-CE23-0011).
A large part of this research was done while R.H. was at Aix-Marseille University; the part in Aalto University in part been funded by Academy of Finland grants 334790 (MAGITICS) and 310107 (MACOME).}


\appendix

\section{Proof for Choi-Kraus representation theorem (Theorem~\ref{th:choiKraus})} \label{app:choiKraus}

	The following result is useful to prove Theorem~\ref{th:choiKraus}.
	
	\begin{theorem}\label{th:watrous}(\citealp[Theorem~2.22]{watrous2018theory})\\
Let $\Phi$ be a linear map from $\mathcal{L(K)}$ to $\mathcal{L(Y)}$, where $\mathcal{K}$  and $\mathcal{Y}$ are Euclidean spaces. The following statements are equivalent:
\begin{enumerate}
	\item There exists an operator $A\in\mathcal{L(K,Y\otimes Z)}$ for a some choice of an Euclidean space $\mathcal{Z}$, such that \[\Phi(\x) = \mathrm{tr}_{\mathcal{Z}}(AXA^*)\] for all $X\in\mathcal{X}$.
	\item There existes a collection $\{A_a : a\in\Sigma\}$, for some choice of an alphabet $\Sigma$, for which \[\Phi(\x) = \sum_a A_a X A_a^*\] for all $X\in\mathcal{X}$.
\end{enumerate}
	\end{theorem}
As mentioned in~\citet{watrous2018theory}, this theorem is an amalgamation of results that are generally attributed to~\citet{stinespring1955positive, kraus1971general,kraus1983states, choi1975completely}, and presents only the finite-dimensional analogues of the results they proved which hold for infinite-dimensional
spaces.

	%
	\paragraph{Proof of Theorem~\ref{th:choiKraus}}
	The proof follows the same arguments as the proof of Theorem~6.5 from~\citet{attalLectures}, which holds for infinite dimensional spaces. It is based on the observation that for $\mathbf{v}\in\mathcal{Y}$ we have $\mathbf{v}\mathbf{v}^\top \otimes (\phi(\x)  \phi(\x)^\top) = \tilde{O}_{\mathbf{v}}(\phi(\x)  \phi(\x)^\top)\tilde{O}_{\mathbf{v}}^*$, in which $\tilde{O}_{\mathbf{v}}$ is the operator defined by
	\begin{align*}
	\tilde{O}_{\mathbf{v}} : \mathcal{K} &\longrightarrow \mathcal{Y} \otimes \mathcal{K}\\
	u &\longmapsto \mathbf{v} \otimes u,
	\end{align*}
	and $\tilde{O}_{\mathbf{v}}^*$ is its adjoint defined by
	\begin{align*}
	\tilde{O}_{\mathbf{v}}^* :  \mathcal{Y} \otimes \mathcal{K} &\longrightarrow \mathcal{K} \\
	\mathbf{y} \otimes u  &\longmapsto \langle \mathbf{v}, \mathbf{y} \rangle u.
	\end{align*}
	The entangled kernel $K$ can now be written as follows
	\begin{align*}
		K(\x,\z) &= \mathrm{tr}_{\mathcal{K}}\left(\mathbf{U} \big(\mathbf{T} \otimes (\phi(\x)  \phi(\z)^\top)\big)\mathbf{U}^\top\right) \\
		& = \sum_t \mathrm{tr}_{\mathcal{K}}\left(\mathbf{U} \big(\mathbf{v}_t \mathbf{v}_t^\top \otimes (\phi(\x)  \phi(\z)^\top)\big)\mathbf{U}^\top\right)\\
		&= \sum_t \mathrm{tr}_{\mathcal{K}}\left(\mathbf{U} \big(\tilde{O}_{\mathbf{v}_t} (\phi(\x)  \phi(\z)^\top) \tilde{O}_{\mathbf{v}_t}^* \big)\mathbf{U}^\top\right)\\
		&= \sum_t \mathrm{tr}_{\mathcal{K}}\left(S_t (\phi(\x)  \phi(\z)^\top) S_t^* \right),
	\end{align*}
	where $S_t=\mathbf{U}\tilde{O}_{\mathbf{v}_t}$.
	Using Theorem~\ref{th:watrous}, which is also valid for infinite-dimensional spaces~\citep{watrous2018theory}, we obtain that there exists a set of matrices $\{\mathbf{M}_i, 1\leq i \leq r\}$ for which $K(\x, \z) =  \sum_{i=1}^r \mathbf{M}_i \phi(\x) \phi(\z)^\top \mathbf{M}_i^\top$. This completes the proof.
	
	${}$ \hfill $\blacksquare$

\section{Proof of Theorem~\ref{th:EKL_Rademacher} (Rademacher bound for EKL)}
\label{app:Rademacher}

We provide here the proof for our Rademacher complexity bound. 

\paragraph{Proof of Theorem~\ref{th:EKL_Rademacher}}
We start by recalling that the feature map associated to the operator-valued kernel $K$ is the mapping $\Gamma: \mathcal{X} \to \mathcal{L}(\mathcal{Y}, \mathcal{H})$, where $\mathcal{X}$ is the input space, $\mathcal{Y}=\mathbb{R}^p$, and $\mathcal{L}(\mathcal{Y}, \mathcal{H})$ is the set of bounded linear operators from $\mathcal{Y}$ to $\mathcal{H}$ (see, e.g., ~\cite{Micchelli2005onlearning,Carmeli2010vector} for more details). It is known that $K(\x,\z) = \Gamma(\x)^*\Gamma(\z)$. We denote by $\Gamma_\mathbf{Q}$ the feature map associated to our entangled kernel~(Equation~\ref{eq:eklKernel}). We also define the matrix $\boldsymbol{\Sigma} = (\boldsymbol{\sigma})_{i=1}^n \in \mathbb{R}^{np}$

\begin{align*}
\hat{\mathcal{R}}_n( \mathcal{H}_\beta) &= \frac{1}{n} \E\left[\sup_{f\in\mathcal{H}} \sup_{\mathbf{Q}\in \Delta}\sum_{i=1}^n \boldsymbol{\sigma}_i^\top f_{u,\mathbf{Q}}(\x_i)\right] \\
&= \frac{1}{n} \E\left[\sup_{u} \sup_{\mathbf{Q}}\sum_{i=1}^n \langle\boldsymbol{\sigma}_i , \Gamma_\mathbf{Q}(\x_i)^* \; u\rangle_{\mathbb{R}^p}\right]\\
& =  \frac{1}{n} \E\left[\sup_{u} \sup_{\mathbf{Q}}\sum_{i=1}^n \langle \Gamma_\mathbf{Q}(\x_i) \boldsymbol{\sigma}_i ,  u\rangle_{\mathcal{H}}\right] \text{\quad (1)}\\
\end{align*}
\begin{align*}
& \leq \frac{\beta}{n}  \E\left[ \sup_{\mathbf{Q}}\|\sum_{i=1}^n  \Gamma_\mathbf{Q}(\x_i) \boldsymbol{\sigma}_i\|_{\mathcal{H}}\right] \text{\quad (2)}\\
& = \frac{\beta}{n}  \E\left[ \sup_{\mathbf{Q}} \left(\sum_{i,j=1}^n \langle \boldsymbol{\sigma}_i,  K_{\mathbf{Q}}(\x_i,\x_j) \boldsymbol{\sigma}_j\rangle_{\mathbb{R}^p}\right)^{\tfrac{1}{2}}\right] \text{\quad (3)}\\
& =  \frac{\beta}{n}  \E\left[ \sup_{\mathbf{Q}} \left( \langle \boldsymbol{\Sigma},  \mathbf{G}_{\mathbf{Q}}\boldsymbol{\Sigma}\rangle_{\mathbb{R}^{np}}\right)^{1/2}\right] \\
& = \frac{\beta}{n}  \E\left[ \sup_{\mathbf{Q}}  \langle \boldsymbol{\Sigma}, [\mathbf{\Phi}^\top \otimes \mathbf{I}_p] \mathbf{QQ^\top } [\mathbf{\Phi}\otimes \mathbf{I}_p]\boldsymbol{\Sigma}\rangle^{1/2}\right] \\
& = \frac{\beta}{n}  \E\left[ \sup_{\mathbf{Q}} tr([\mathbf{\Phi} \otimes \mathbf{I}_p] \boldsymbol{\Sigma}\boldsymbol{\Sigma}^\top[\mathbf{\Phi}^\top \otimes \mathbf{I}_p]\mathbf{QQ}^\top)^{1/2}\right] \\
& \leq \frac{\beta}{n}  \E\left[ \sup_{\mathbf{Q}} tr([[\mathbf{\Phi} \otimes \mathbf{I}_p] \boldsymbol{\Sigma}\boldsymbol{\Sigma}^\top[\mathbf{\Phi}^\top \otimes \mathbf{I}_p]]^2)^{1/4} tr(\mathbf{QQ}^\top\mathbf{QQ}^\top)^{1/4}\right] \text{\quad (4)} \\
& \leq \frac{\beta}{n}  \E\left[ \sup_{\mathbf{Q}} tr([\mathbf{\Phi}\otimes \mathbf{I}_p][\mathbf{\Phi}^\top \otimes \mathbf{I}_p] \boldsymbol{\Sigma}\boldsymbol{\Sigma}^\top)^{1/2} tr(\mathbf{QQ}^\top\mathbf{QQ}^\top)^{1/4}\right]\\
& \leq \frac{\beta}{n}  \E\left[ \sup_{\mathbf{Q}} tr([\mathbf{\Phi}\mathbf{\Phi}^\top \otimes \mathbf{I}_p] \boldsymbol{\Sigma}\boldsymbol{\Sigma}^\top)^{1/2} \right]\\
& = \frac{\beta}{n}  \E\left[tr([\mathbf{K} \otimes \mathbf{I}_p] \boldsymbol{\Sigma}\boldsymbol{\Sigma}^\top)^{1/2} \right] \\
& \leq \frac{\beta}{n}  \left(\E\left[tr([\mathbf{K} \otimes \mathbf{I}_p] \boldsymbol{\Sigma}\boldsymbol{\Sigma}^\top)\right]\right)^{1/2}  \text{\quad (5)} \\
& = \frac{\beta}{n}  \left(tr\left[[\mathbf{K} \otimes \mathbf{I}_p]\E(\boldsymbol{\Sigma}\boldsymbol{\Sigma}^\top)\right]\right)^{1/2}   \\
& =  \frac{\beta}{n}  \sqrt{tr(\mathbf{K})p}. 
\end{align*}
Here (1) and (3) are obtained with reproducing property, (2) and (4) with Cauchy-Schwarz inequality, and (5) with Jensen's inequality.

For kernels $k$ that satisfy $tr(\mathbf{K}) \leq \kappa n $, we obtain that 
$$\hat{\mathcal{R}}_n( \mathcal{H}_\beta) \leq \frac{\beta \sqrt{\kappa p}}{\sqrt{n}}.  $$

${}$ \hfill $\blacksquare$



\section{Proof of Corollary \ref{th:EKL_Generalization} (Generalization bound for EKL)}
\label{app:Generalization}

The following two results are useful to prove  Corollary~\ref{th:EKL_Generalization}.
 First, \citet{bartlett2002rademacher} provides the following generalization bound based on Rademacher complexity (see also \citealt[Theorem~3.3]{mohri2018foundations}).  
\begin{theorem}\label{th:Rademacher_Generalization} (\citealt[Theorem~3.3]{mohri2018foundations}) \\[0.1cm]
	Let $\mathcal{G}$ be a family of functions mapping from an arbitrary input space $\mathcal{Z}$ to [0,M]. Then, for any $\delta >0$, with probability at least $1-\delta$ over the draw of an i.i.d. sample $S$ of size $n$, the following holds for all $g\in\mathcal{G}$:
	\begin{equation}
	\label{eq:Rademacher_Generalization}
		\mathbb{E}[g(\z)] \leq \frac{1}{n}\sum_{i=1}^n g(\z_i) + 2 \hat{\mathcal{R}}_n(\mathcal{G}) + 3 M \sqrt{\frac{\log\frac{2}{\delta}}{2n}},
	\end{equation}
	where $\hat{\mathcal{R}}_n(\mathcal{G})$ is the empirical Rademacher complexity of $\mathcal{G}$.
\end{theorem}
The second result, from~\citet{maurer2016vector}, provides a contraction inequality for the Rademacher complexity of classes of vector-valued functions.
\begin{corollary}\label{th:vector_contraction_rademacher}
	(\citealt[Corollary~1]{maurer2016vector}) \\[0.1cm]
Let $\mathcal{X}$ be any set, $(\x_1,\ldots,\x_n)\in \mathcal{X}^n$, let $\mathcal{F}$ be a class of functions  $f:\mathcal{X} \to \ell_2$ and let $h_i : \ell_2 \to \mathbb{R}$ have Lipschitz norm $L$. Then
\begin{equation}
\label{eq:vector_contraction_rademacher}
\mathbb{E}\sup_{f\in\mathcal{F}}\sum_i \sigma_i h_i(f(\x_i)) \leq \sqrt{2} L \mathbb{E} \sup_{f\in\mathcal{F}}\sum_{i,j=1}\sigma_{ij} f_j(\x_i),
\end{equation}
where $\sigma_{ij}$ is an independent doubly indexed Rademacher sequence and $f_j(\x_i)$ is the j-th component of $f(\x_i)$.
\end{corollary}
We now make use of the above results to prove the generalization error bound of EKL.

\subsection*{Proof of  Corollary \ref{th:EKL_Generalization}}
Since $\|f(\x)-\mathbf{y}\|_{\mathcal{Y}} \leq M$ for all $(\x,\mathbf{y})\in \mathcal{X} \times \mathcal{Y}$ and $f\in\mathcal{H}_\beta$, for any $\mathbf{y}'$ the function $y\mapsto \|\mathbf{y}-\mathbf{y}'\|_{\mathcal{Y}}^2$ is $2 M$-Lipschitz. Then by Corollary~\ref{th:vector_contraction_rademacher}, for any sample $S = ((\x_1,\mathbf{y}_1),\ldots,(\x_n,\mathbf{y}_n))$, the Rademacher complexity of
the family $\mathcal{G} = \{(\x,\mathbf{y})\mapsto \|f(\x)-\mathbf{y}\|_\mathcal{Y}^2 : f\in\mathcal{H}_\beta\}$  is upper bounded as follows:
\begin{equation}
\hat{\mathcal{R}}_n(\mathcal{G}) \leq 2\sqrt{2} M \hat{\mathcal{R}}_n( \mathcal{H}_\beta).
\end{equation}
Combining this inequality with the general Rademacher complexity learning bound
of Theorem~\ref{th:Rademacher_Generalization} and the Rademacher complexity bound of  $\mathcal{H}_\beta$ given in Theorem~\ref{th:EKL_Rademacher} completes the proof.

${}$ \hfill $\blacksquare$



\vskip 0.2in
\bibliography{entangledkernel}

\end{document}